\pdfoutput = 1
\documentclass[11pt]{article}
\usepackage{fullpage}
\usepackage[english]{babel}
\usepackage[utf8x]{inputenc}
\usepackage[T1]{fontenc}
\usepackage{amsmath}
\usepackage{amssymb}
\usepackage{amsthm}
\usepackage{bbm}
\usepackage{bbold}
\usepackage{parskip}
\usepackage{thm-restate}
\usepackage{booktabs}
\usepackage{array}
\usepackage{xcolor}

\usepackage{pifont}
\newcommand{\cmark}{\ding{51}}
\newcommand{\xmark}{\ding{55}}

\usepackage{algorithm}
\usepackage{algpseudocode}

\usepackage{tikz}

\usepackage{hyperref}
\hypersetup{
	colorlinks=true,
	linkcolor=blue!70!black,
	citecolor=blue!70!black,
	urlcolor=blue!70!black
}
\usepackage{cleveref}
\usepackage{graphicx}
\usepackage{algorithm}
\usepackage[lined,boxed,ruled,norelsize,algo2e,linesnumbered]{algorithm2e}
\hypersetup{
	colorlinks=true,
	linkcolor=blue!70!black,
	citecolor=blue!70!black,
	urlcolor=blue!70!black
}

\newtheorem{theorem}{Theorem}
\newtheorem{lemma}{Lemma}

\newtheorem{definition}{Definition}
\newtheorem{corollary}{Corollary}
\newtheorem{proposition}{Proposition}
\newtheorem{remark}{Remark}
\newtheorem{model}{Model}
\newtheorem{problem}{Problem}
\newtheorem{construction}{Construction}

\renewcommand{\epsilon}{\varepsilon}

\newcommand{\vlam}{\boldsymbol{\lambda}}
\newcommand{\defeq}{:=}

\newcommand{\norm}[1]{\left\lVert#1\right\rVert}

\newcommand{\inprod}[2]{\left\langle#1, #2\right\rangle}

\newcommand{\eps}{\varepsilon}
\newcommand{\lam}{\lambda}

\newcommand{\R}{\mathbb{R}}

\newcommand{\N}{\mathbb{N}}

\newcommand{\half}{\frac{1}{2}}

\newcommand{\E}{\mathbb{E}}

\definecolor{burntorange}{rgb}{0.8, 0.33, 0.0}

\newcommand{\Par}[1]{\left(#1\right)}
\newcommand{\Brack}[1]{\left[#1\right]}
\newcommand{\Brace}[1]{\left\{#1\right\}}
\newcommand{\Abs}[1]{\left|#1\right|}

\newcommand{\oracle}{\mathcal{O}}
\newcommand{\alg}{\mathsf{alg}}
\newcommand{\ind}{\mathbb{I}}

\newcommand{\lsq}{\ell_{\mathsf{sq}}}
\newcommand{\llog}{\ell_{\mathsf{log}}}

\newcommand{\va}{\mathbf{a}}

\newcommand{\vc}{\mathbf{c}}
\newcommand{\vd}{\mathbf{d}}
\newcommand{\ve}{\mathbf{e}}
\newcommand{\vf}{\mathbf{f}}
\newcommand{\vg}{\mathbf{g}}

\newcommand{\vi}{\mathbf{i}}

\newcommand{\vk}{\mathbf{k}}

\newcommand{\vp}{\mathbf{p}}
\newcommand{\vq}{\mathbf{q}}

\newcommand{\vu}{\mathbf{u}}
\newcommand{\vv}{\mathbf{v}}
\newcommand{\vw}{\mathbf{w}}
\newcommand{\vx}{\mathbf{x}}
\newcommand{\vy}{\mathbf{y}}

\newcommand{\0}{\mathbf{0}}
\newcommand{\1}{\mathbf{1}}

\newcommand{\calA}{\mathcal{A}}
\newcommand{\calB}{\mathcal{B}}

\newcommand{\calD}{\mathcal{D}}
\newcommand{\calE}{\mathcal{E}}
\newcommand{\calF}{\mathcal{F}}
\newcommand{\calG}{\mathcal{G}}
\newcommand{\calH}{\mathcal{H}}

\newcommand{\calK}{\mathcal{K}}

\newcommand{\calM}{\mathcal{M}}
\newcommand{\calN}{\mathcal{N}}

\newcommand{\calQ}{\mathcal{Q}}
\newcommand{\calR}{\mathcal{R}}

\newcommand{\calU}{\mathcal{U}}
\newcommand{\calV}{\mathcal{V}}

\newcommand{\calX}{\mathcal{X}}
\newcommand{\calY}{\mathcal{Y}}

\newcommand{\xset}{\calX}

\newcommand{\reg}{\textup{reg}}

\newcommand{\codeStyle}[1]{{\bfseries #1} }

\newcommand{\codeReturn}{\codeStyle{Return:}}

\newcommand{\tvg}{\tilde{\vg}}

\newcommand{\sps}[1]{^{(#1)}}

\renewcommand{\epsilon}{\varepsilon}

\newcommand{\veps}{\boldsymbol{\eps}}
\newcommand{\bsa}{\mathsf{ImbalancedSimultaneousApproach}}
\newcommand{\Bern}{\mathsf{Bern}}

\newtheorem{assumption}{Assumption}
\newcommand{\by}{\bar{y}}
\newcommand{\hcalR}{\widehat{\calR}}
\newcommand{\sfL}{\mathsf{L}}

\usepackage{mathtools}
\usepackage{amsbsy}

\title{Optimal Recalibration of an Online Predictor}
\author{
Lunjia Hu\thanks{Northeastern University, \texttt{lunjia@alumni.stanford.edu}} \and 
Kevin Tian\thanks{University of Texas at Austin, \texttt{kjtian@cs.utexas.edu}} \and 
Chutong Yang\thanks{University of Texas at Austin, \texttt{cyang98@utexas.edu}}
}
\date{}
\allowdisplaybreaks
\begin{document}

\maketitle
\begin{abstract}
We study the problem of recalibrating an online predictor \cite{KuleshovE17, OkoroaforKS24}: given an arbitrary ``hint'' sequence of forecasts, the learner must output new predictions that are calibrated while incurring small excess error relative to the original forecasts, under a proper loss. We give an online algorithm that achieves 
$(\eps, \eps^2)$-recalibration for Lipschitz proper losses in 
 $T \approx \eps^{-3}$ rounds, using an imbalanced extension of the recent simultaneous Blackwell approachability reduction framework of \cite{HuTY26}. We show that this tradeoff is optimal by proving a matching lower bound for recalibrating against the squared loss. We also prove a companion $\calK_2$-recalibration theorem that obtains the same tradeoffs up to a logarithmic factor. 
 
 As our main application, we show how our recalibration algorithms can be combined with the online refinement method of \cite{FosterH23} to obtain simultaneous 
$\eps$-calibration and 
$\eps^2$-calibeating for smooth proper losses at the same asymptotic rate, improving upon prior works that achieved these properties separately or with a worse $\eps$ dependence.  In particular, the $\calK_2$ variant answers a question of \cite{ChenHJL26} on simultaneously achieving near-optimal calibeating and calibration rates. We also derive extensions to settings with multiple hint sequences. Finally, we empirically evaluate our algorithms on a classification dataset undergoing distribution shift.
\end{abstract}

\thispagestyle{empty}
\newpage
\tableofcontents
\thispagestyle{empty}
\newpage

\section{Introduction}\label{sec:intro}
\setcounter{page}{1}

Probabilistic prediction is a central primitive in modern machine learning. In an online classification setting with binary labels, this problem is as follows: examples arrive over $T$ timesteps, and for each $t \in [T]$ our goal is to output predictions $p_t \in [0, 1]$ (possibly depending on auxiliary example features) before seeing the label $y_t \in \{0, 1\}$.
These predictions $p_{[T]} \defeq \{p_t\}_{t \in [T]}$ are routinely interpreted as
probabilities modeling the labels $y_{[T]} \defeq \{y_t\}_{t \in [T]}$: an image classifier assigns probabilities to classes, a recommender predicts the probability of a click or purchase, and so on. 

To measure the interpretability of predictions $p_{[T]}$ as a reflection of the labels $y_{[T]}$, calibration \cite{Dawid82, FosterV98} has arisen as a standard desideratum. For prediction-label pairs $p_{[T]}, y_{[T]}$, denote the average label and total count of a predicted value $v \in [0, 1]$ up to iteration $t \in [T]$ by
\[
\Bar{y}_t(p) \defeq \frac{1}{n_t(p)}\sum_{s\in[t]}y_s\cdot\ind_{p_s = p},
\qquad
n_t(p) \defeq \sum_{s\in[t]} \ind_{p_s = p}.
\]
with the conventions $n_0(v)=0$, $\by_0(v)=0$, and $\by_t(v) = 0$ whenever $n_t(v) = 0$. Then the calibration error\footnote{Definition~\ref{def:calerr} is sometimes called the \emph{expected calibration error} (ECE) in the literature, to disambiguate with other calibration metrics.} is defined as follows.

\begin{definition}[Calibration error]\label{def:calerr}
For a sequence of predictor-label pairs $p_{[T]} \in [0, 1]^T$, $y_{[T]} \in \{0, 1\}^T$, we define the \emph{calibration error} by\footnote{Although the sum is over an infinite domain, only finitely many $v \in [0, 1]$ have nonzero $n_T(v)$.}
\begin{equation}\label{eq:l1cal}\calK_1\Par{p_{[T]}, y_{[T]}} \defeq \sum_{v \in [0, 1]} \frac{n_T(v)} T |v - \by_T(v)|. \end{equation}

We also define the \emph{$\calK_2$-calibration error}:
\begin{equation}\label{eq:l2cal}
\calK_2\Par{p_{[T]}, y_{[T]}} \defeq \sqrt{\sum_{v \in [0, 1]} \frac{n_T(v)}{T}\Par{v - \by_T(v)}^2}.
\end{equation}
\end{definition}
For any prediction-label sequences $(p_{[T]}, y_{[T]})$, these two calibration measures satisfy
\[
\calK_2^2\Par{p_{[T]}, y_{[T]}} \le \calK_1\Par{p_{[T]}, y_{[T]}} \le \calK_2\Par{p_{[T]}, y_{[T]}}.
\]

Note that if $\calK_1(p_{[T]}, y_{[T]}) = 0$, then $\by_T(v) = v$ for all predicted $v$; this reflects the interpretable condition that a $v$ fraction of iterations predicting $p_t = v$ will have labels $y_t = 1$. In this case, the predictions satisfy the classical definition of exact \emph{calibration} \cite{Dawid82, FosterV98}.

Oftentimes, calibration is only treated as a basic requirement, alongside other criteria such as \emph{loss minimization} (e.g., as measured by the squared loss). Indeed, calibration and loss minimization are generally incomparable; a predictor can be extremely miscalibrated while achieving near-optimal loss minimization, and vice versa. In this paper, we study the problem of \emph{recalibrating} an online predictor, which balances these two goals in a general \emph{hinted sequential prediction} setting.

\begin{model}[Hinted sequential prediction]\label{model:hsp}
In the \emph{hinted sequential prediction} setting, a learner observes a \emph{hint} sequence $q_{[T]} \in [0, 1]^T$, and outputs a \emph{prediction} sequence $p_{[T]} \in [0, 1]^T$ before observing a \emph{label} sequence $y_{[T]} \in \{0, 1\}^T$. We assume $(q_t, y_t)$ can depend jointly on all previous $(q_{[t - 1]}, p_{[t - 1]}, y_{[t - 1]})$, but not $p_t$; similarly, $p_t$ can depend on $(q_{[t - 1]}, p_{[t - 1]}, y_{[t - 1]})$ and $q_t$, but not $y_t$.

This setting can also be viewed as a $T$-iteration game between nature and a learner. In the $t^{\text{th}}$ iteration, nature selects a hint-label pair $(q_t, y_t)$ (depending on the game's history) and reveals $q_t$ to the learner. The learner then makes a prediction $p_t$, and then nature reveals $y_t$ to the learner.
\end{model}

In Model~\ref{model:hsp}, the hint $q_{[T]}$ can be viewed as an initial prediction, e.g., made by a model with good loss properties, but that does not necessarily satisfy calibration. This tension is motivated by \cite{GuoPSW17, KadavathCAHDPSHDT22}, studies which demonstrate that modern neural networks trained with a loss minimization objective are not always well-calibrated. Thus, a natural goal under Model~\ref{model:hsp} is to observe the hint $q_{[T]}$, and output a calibrated sequence $p_{[T]}$ that achieves similar loss. This is the goal in the problem of \emph{recalibration}, introduced by \cite{KuleshovE17} and studied further by \cite{OkoroaforKS24}.

\begin{definition}[Recalibration]\label{def:recal}
In the setting of Model~\ref{model:hsp}, let $\alpha, \beta \ge 0$, and let $\ell: [0, 1] \times \{0, 1\} \to \R$ be a loss. We say that (possibly randomized) $p_{[T]}$ is $(\alpha,\beta)$-recalibrated against
$(q_{[T]}, y_{[T]})$ under $\ell$ if
\[
\E\Brack{\calK_1\Par{p_{[T]},y_{[T]}}}\le \alpha,\quad\text{ and } \quad
\E\Brack{\ell\Par{p_{[T]},y_{[T]}}
-
\ell\Par{q_{[T]},y_{[T]}}}\le \beta.
\]
If only the first condition holds, we say that $p_{[T]}$ is $\alpha$-calibrated against $y_{[T]}$. Similarly, we say that (possibly randomized) $p_{[T]}$ is $(\alpha,\beta)$-$\calK_2$-recalibrated against
$(q_{[T]}, y_{[T]})$ under $\ell$ if
\[
\E\Brack{\calK_2\Par{p_{[T]},y_{[T]}}}\le \alpha,\quad\text{ and } \quad
\E\Brack{\ell\Par{p_{[T]},y_{[T]}}
-
\ell\Par{q_{[T]},y_{[T]}}}\le \beta.
\]
If only the first condition holds, we say that $p_{[T]}$ is $\alpha$-$\calK_2$-calibrated against $y_{[T]}$. 
\end{definition}

In Definition~\ref{def:recal}, we overload notation: for a loss $\ell$ that takes prediction-label pairs $(p, y)$ to a real value, we define its application to sequences by $\ell(p_{[T]}, y_{[T]}) \defeq \frac 1 T \sum_{t \in [T]} \ell(p_t, y_t)$. In this introduction, we focus primarily on the \emph{squared loss} $\ell(p, y) = \lsq(p, y) \defeq (p - y)^2$ for simplicity, although our main results apply to a much broader range of \emph{proper losses} (cf.\ Section~\ref{sec:prelims} for definitions).

\subsection{Our results}\label{ssec:results}

Our main result is an optimal algorithm for recalibration with parameters $(\eps, \eps^2)$. In particular, we provide a matching lower bound, showing our required iteration complexity is tight.

\begin{theorem}[Informal, see \Cref{thm:recalibration,thm:lb}]\label{thm:recal_inf}
If $T = \Omega(\eps^{-3})$, there is an algorithm that achieves $(\eps,\eps^2)$-recalibration under $\lsq$ in the setting of Model~\ref{model:hsp}, which is optimal up to constant factors.
\end{theorem}

\begin{table}[t]
\centering
\renewcommand{\arraystretch}{1.15}
\setlength{\tabcolsep}{8pt}
\begin{tabular}{lc}
\toprule
\textbf{Work} & \textbf{Rate} \\
\midrule
\cite{KuleshovE17}   & $O\Par{\frac{1}{\eps^8}}$ \\
\cite{OkoroaforKS24} & $O\Par{\frac{1}{\eps^4}}$ \\
\addlinespace[2pt]
\cmidrule(lr){1-2}
\addlinespace[2pt]
Theorems~\ref{thm:recalibration},~\ref{thm:lb}
& $\Theta\Par{\frac{1}{\eps^3}}$ \\
\bottomrule
\end{tabular}
\caption{Rates for $(\eps,\eps^2)$-recalibration under $\lsq$.}
\label{tab:comparison_compact}
\end{table}

We remark that our formal upper bound in Theorem~\ref{thm:recalibration} extends to all $L$-Lipschitz proper losses, with a polynomial dependence on $L$ in the iteration count.
In Table~\ref{tab:comparison_compact}, we compare Theorem~\ref{thm:recal_inf} with prior related results from \cite{KuleshovE17, OkoroaforKS24}. Our result improves known bounds on a question posed by \cite{OkoroaforKS24}, on the achievable exponent pairs $(\alpha, \beta) = (T^{-a}, T^{-b})$ in recalibration.

We also prove a $\calK_2$ analog of our recalibration upper bound, that matches Theorem~\ref{thm:recalibration} up to a logarithmic factor. This result uses a different algorithmic primitive from the $\calK_1$ theorem above, based on bounding both $\calK_2$-calibration and excess loss via a family of quadratic tests.

\begin{theorem}[Informal, see \Cref{thm:l2recalibration}]\label{thm:l2_recal_inf}
If $T = \widetilde{\Omega}(\eps^{-3})$, there is an algorithm that achieves $(\eps,\eps^2)$-$\calK_2$-recalibration under $\lsq$ in the setting of Model~\ref{model:hsp}, which is optimal up to logarithmic factors.
\end{theorem}
Similarly to the $\calK_1$ case, our formal \Cref{thm:l2recalibration} extends to all Lipschitz proper losses.

The horizon $T = \Omega(\eps^{-3})$ in Theorem~\ref{thm:recal_inf} is natural from the perspective of the simpler problem of \emph{online calibration}, where the only goal is to achieve $\E[\calK_1(p_{[T]}, y_{[T]})] \le \eps$. In this related setting, $\eps$-calibration was known to be achievable in $T = \Omega(\eps^{-3})$ iterations since \cite{FosterV98}, but improving upon this tradeoff was open for nearly 30 years until a recent breakthrough of \cite{DaganDFGKO25}. Our work shows that for the stronger goal of $(\eps, \eps^2)$-recalibration, taking a horizon of $T = \Omega(\eps^{-3})$ is not only sufficient, it is also necessary, establishing a strict separation with online calibration.

Although our lower bound shows that no shorter horizon is possible for recalibration with parameters $(\eps, \eps^2)$, the excess loss requirement of $\eps^2$ may seem artificial. Indeed, if calibration is not a consideration, achieving $\beta = 0$ in Definition~\ref{def:recal} is trivial using $p_{[T]} = q_{[T]}$. Our next application, \emph{calibeating}, motivates this particular $(\eps, \eps^2)$ tradeoff. This problem was introduced by \cite{FosterH23} and expanded upon by \cite{ChenHJL26}, where the goal is stronger than merely not worsening the loss of a hint sequence $q_{[T]}$: instead, its loss must be improved up to the \emph{refinement score} \eqref{eq:refinement}.

\begin{definition}[Calibeating]\label{def:calibeating}
In the setting of Model~\ref{model:hsp}, let $\alpha \ge 0$, and let $\ell: [0, 1] \times \{0, 1\} \to \R$ be a loss. We say that (possibly randomized) $p_{[T]}$ \emph{$\alpha$-calibeats} $(q_{[T]}, y_{[T]})$ under $\ell$ if
\[\E\Brack{\ell\Par{p_{[T]}, y_{[T]}}} \le \E\Brack{\calR_\ell\Par{q_{[T]}, y_{[T]}}} + \alpha,\]
where we define the \emph{$\ell$-refinement score} by
\begin{equation}\label{eq:refinement}\calR_\ell\Par{q_{[T]}, y_{[T]}} \defeq \frac 1 T \sum_{t \in [T]} \ell\Par{\by_T(q_t), y_t}.\end{equation}
\end{definition}

We write
\[
s(q_{[T]})\defeq\Abs{\Brace{q_t:t\in[T]}}
\]
for the number of distinct hint values, and say that $q_{[T]}$ is \emph{$s$-sparse} if $s(q_{[T]})\le s$.

The refinement score \eqref{eq:refinement} is the loss achieved by the \emph{hindsight average predictor} $\by_T(q_t)$, which can observe the label sequence $y_{[T]}$ in advance and output corrected label averages.  Calibeating, i.e., competing with $\by_T(q_t)$ in an online fashion, was motivated by the classical bias-variance decomposition of the Brier score (the loss $\lsq$). Note that this benchmark predictor is infeasible to implement in the online setting of Model~\ref{model:hsp}. Perhaps surprisingly, \cite{FosterH23} showed that competing with the refinement score up to $\alpha = \eps^{2}$ remains asymptotically achievable even in the \emph{online setting} of Model~\ref{model:hsp}, at the same $\approx \eps^{-3}$ scale, as long as $q_{[T]}$ takes on at most $O(\frac 1 \eps)$ distinct values.\footnote{We discuss this requirement further in Remark~\ref{rem:sparsity}, which explains why nontrivial calibeating is impossible without a sparsity assumption on the hint sequence $q_{[T]}$. The regime $O(\frac 1 \eps)$ is natural for predictions on a grid of width $\eps$.}

Interestingly, although the hindsight average predictor $\by_T(q_t)$ is perfectly-calibrated by definition, the online calibeating algorithm of \cite{FosterH23} does not by itself guarantee calibration. Assume the reference sequence $q_{[T]}$ is $s = O(\frac 1 \eps)$ sparse. The state-of-the-art prior result was that $\eps$-calibration and $\eps^2$-calibeating  for $\lsq$ were achievable in isolation with $T  = \widetilde O(\eps^{-3})$, but not simultaneously. The conclusion section of \cite{ChenHJL26} explicitly asks whether having $T \approx \eps^{-3}$ rounds is sufficient to simultaneously achieve the optimal $T^{-1/3} \approx \eps$ calibration rate and $\frac{s\log T}{T} \approx  \eps^2$ calibeating rate, matching the best known bounds for the two objectives separately, and whether analogous simultaneous guarantees extend beyond the Brier loss.

Our formal results Theorems~\ref{thm:proper_loss_main} and \ref{thm:l2_calibeating} answer this question affirmatively. These results apply to general Lipschitz proper losses and all sparsity levels. The following informal statement is the special case for the Brier loss $\lsq$ and sparsity $s = O(\frac 1 \eps)$:

\begin{theorem}[Informal, see \Cref{thm:proper_loss_main,thm:l2_calibeating}]\label{thm:calibeat_inf} If $T = \widetilde{\Omega}(\eps^{-3})$ and $q_{[T]}$ is $s$-sparse for $s=O(\frac 1 \eps)$, there is an algorithm that achieves $\eps$-$\calK_2$-calibration, and hence $\eps$-calibration, together with $\eps^2$-calibeating under $\lsq$ in the setting of Model~\ref{model:hsp}.
\end{theorem}

Theorem~\ref{thm:calibeat_inf} is proven by combining a generalization of the main calibeating result of \cite{FosterH23} (see Proposition~\ref{prop:online_offline_gap}) with our $\calK_1$ and $\calK_2$ recalibration results (Theorems~\ref{thm:recal_inf} and~\ref{thm:l2_recal_inf}). It extends beyond $\lsq$ to the wider family of Lipschitz proper losses (Assumption~\ref{assume:lip}). The formal variants in \Cref{thm:proper_loss_main,thm:l2_calibeating} state the guarantee for arbitrary sparsity $s$, with polynomial dependences on the smoothness parameters and an explicit linear dependence on $s$. We also provide a comparison with prior related works in Table~\ref{tab:comparison_calibeat}.

\begin{table}[t]
\centering
\renewcommand{\arraystretch}{1.15}
\setlength{\tabcolsep}{8pt}
\begin{tabular}{lcc}
\toprule
\textbf{Work} & \textbf{Rate} & \textbf{Multiple sequences?} \\
\midrule
\cite{FosterH23}   & $\widetilde{O}\Par{\frac{1}{\eps^4}}$ & \xmark \\
\cite{LeeNPR22}    & $\widetilde{O}\Par{\frac{1}{\eps^6}}$ & \cmark \\
\cite{HuLSS2025efficient}& $\widetilde{O}\Par{\frac{1}{\eps^6}}$ & \cmark \\
\cite{ChenHJL26}   & $\widetilde{O}\Par{\frac{1}{\eps^4}}$ & \cmark \\
\addlinespace[2pt]
\cmidrule(lr){1-3}
\addlinespace[2pt]
Theorems~\ref{thm:proper_loss_main} and~\ref{thm:multi-proper-loss-main}
& $\widetilde{O}\Par{\frac{1}{\eps^3}}$ & \cmark \\
\bottomrule
\end{tabular}
\caption{Rates for $\eps$-calibration and $\eps^2$-calibeating under $\lsq$, assuming $q_{[T]}$ is $O(\frac 1 \eps)$-sparse. }
\label{tab:comparison_calibeat}
\end{table}

Finally, in \Cref{sec:multi,sec:l2-multi}, we show that Theorems~\ref{thm:recal_inf}, \ref{thm:l2_recal_inf} and~\ref{thm:calibeat_inf} generalize naturally to the setting of Model~\ref{model:mhsp}, a variant of Model~\ref{model:hsp} where there are $m$ hint sequences, and the goal is to compete with the loss of the best hint sequence (in hindsight), or the appropriate refinement score extension. This problem setting is natural when a portfolio of models (e.g., with different hyperparameters) is used to generate predictions. By using a reduction technique from \cite{ChenHJL26} that we term \emph{hint compression} (cf.\ Definition~\ref{problem:hint}), we show in Theorems~\ref{thm:mrecalibration},~\ref{thm:multi-proper-loss-main},~\ref{thm:l2_mrecalibration}, and~\ref{thm:l2_multi_calibeating} that both our $\calK_1$ and $\calK_2$ algorithms extend to the multi-hinted setting at an additive cost of $T = \Omega(\frac{L\log m}{\eps^2})$, suppressing the baseline single-hint terms. This property also holds for the prior works of \cite{LeeNPR22, HuLSS2025efficient, ChenHJL26}, as indicated in Table~\ref{tab:comparison_calibeat}.

\subsection{Technical overview}

In this section, we provide an overview of the proofs of both parts of Theorem~\ref{thm:recal_inf}, as well as the upper bound in Theorem~\ref{thm:l2_recal_inf}. We defer discussions of our other main results to their respective sections, as they are direct applications of Theorems~\ref{thm:recal_inf} and~\ref{thm:l2_recal_inf} combined with reductions present in prior work (appropriately generalized to arbitrary Lipschitz proper losses).

\textbf{Recalibration as imbalanced simultaneous Blackwell approachability.}
Blackwell approachability \cite{Blackwell56} is a classical strategy in online learning, and has a documented connection to regret minimization \cite{AbernethyBH11} and online calibration \cite{FosterV98, Foster99}.
Our recalibration algorithm is based on encoding calibration and loss preservation as two simultaneous Blackwell
approachability objectives. This general strategy of balancing  approachability objectives was the explicit focus of the recent work \cite{HuTY26}, and was implicit in several prior works, including \cite{OkoroaforKS24, OkoroaforKK25}.
Where our work differs is in introducing a variant of the \cite{HuTY26} simultaneous Blackwell approachability framework that explicitly encodes and benefits from asymmetry (Theorem~\ref{thm:sba}). 

We first recall the encoding of recalibration as approachability, as stated in \cite{OkoroaforKS24}.
The learner's action is a distribution \(\va\in\Delta^{\calN}\) over a fine grid
\(\calN\subset[0,1]\), which is later sampled to produce an ordinary scalar prediction. The first approachability objective is induced by
the calibration vector
\begin{equation}\label{eq:v1_intro}
    \vv^{(1)}(\va,q,y)=\{\va_s(s-y)\}_{s\in\calN}.
\end{equation}
Our goal is to force this vector to approach the origin in $\R^{\calN}$, by controlling its (time-averaged) $\ell_1$ norm. Indeed, the $\ell_1$ norm of the average $\vv^{(1)}$ is exactly the calibration error (Definition~\ref{def:calerr}). 

The second approachability objective is induced by the excess
loss
\begin{equation}\label{eq:v2_intro}
    v^{(2)}(\va,q,y)=\sum_{s\in\calN}\va_s(\ell(s,y)-\ell(q,y)),
\end{equation}
whose average is the loss increase relative to the hint, which we want to approach $0$. Thus \((\eps,\eps^2)\)-recalibration is a two-objective approachability problem, but with different target accuracies.

\textbf{Solving imbalanced simultaneous Blackwell approachability.}
We now explain the technical modifications to \cite{HuTY26} that our asymmetry-aware Theorem~\ref{thm:sba} makes, en route to proving the upper bound in Theorem~\ref{thm:recal_inf}. The main result of \cite{HuTY26} can be viewed as a reduction, from solving $m$ simultaneous Blackwell approachability problems with coupled actions, to implementing $m$ disjoint online learning algorithms and a certain oracle. This oracle, termed a \emph{mixture linear optimization oracle} (MLOO, Definition~\ref{def:mloo}), essentially asks to approximately satisfy $m$ halfspace constraints \emph{on average}, where each approachability instance induces one constraint.

The uniform definition of the MLOO in \cite{HuTY26} is incompatible with our asymmetric framing of recalibration. Therefore, in Definition~\ref{def:mloo}, we generalize the definition to handle asymmetric tolerance parameters for the halfspace constraints, with the final tolerance being a weighted average. In Corollary~\ref{cor:sba_imbalanced_e}, we show that this new oracle definition, carefully combined with an imbalanced mirror descent with two experts within the \cite{HuTY26} framework, yields a rate compatible with Theorem~\ref{thm:recal_inf}. More precisely, our mirror descent implementation with step size $\eta \approx \eps^2$ has one expert pay an average regret of $\eta + \frac 1 {\eta T}$, but the other expert only pays $\eta$; this is important because only the calibration objective can afford a $\frac 1 {\eta T} = \eps$ overhead when $T = \Omega(\eps^{-3})$.

The remaining challenge is to implement the asymmetric MLOO for the two objectives
\eqref{eq:v1_intro}, \eqref{eq:v2_intro}. Given a calibration distinguisher \(\vu\in[-1,1]^{\calN}\), mixture weights \((\alpha,\beta)\), and
the current hint \(q\), the oracle must choose a distribution over grid points so that, simultaneously for
both labels \(y\in\{0,1\}\),
\[
    \alpha\sum_{s\in\calN}\va_s u_s(s-y)
    +\beta\sum_{s\in\calN}\va_s(\ell(s,y)-\ell(q,y))
    \le O(\alpha\eps+\beta\eps^2).
\]
Fortunately, such an oracle was provided as Lemma 7 of \cite{OkoroaforKS24} for any $O(1)$-Lipschitz loss $\ell$, although it was not stated that way in the original work. We recall the geometric oracle construction of \cite{OkoroaforKS24} in Lemma~\ref{lem:recal_mloo} in a way that is compatible with our approachability framework, and show it implements an $(\eps, \eps^2)$-MLOO in the sense of Definition~\ref{def:mloo}.

\textbf{$\calK_2$ recalibration via quadratic audits.} 
\Cref{sec:l2recal} gives a companion theorem for $\calK_2$ calibration, via a different route than simultaneous Blackwell approachability, albeit at the cost of a logarithmic factor. The key point is that squared $\calK_2$ has an exact Fenchel representation by quadratic tests of the form $2ue-u^2$, where $e=z-y$ is the calibration residual in a prediction bin. For a proper loss with $G$-Lipschitz link $\psi(p)=\ell(p,0)-\ell(p,1)$, the Savage representation and a one-dimensional co-coercivity inequality show that proper-loss regret is also dominated by these same quadratic tests. We then control the tests by running, in each prediction bin, a square-loss exponential-weights aggregator over a finite audit class containing the constant calibration test and one proper-loss test for each grid prediction. A one-step minimax forecasting oracle keeps the learner's own quadratic payoff at order $\frac 1 {N^2}$ where $N \approx \frac{\sqrt{\max\{1,G\}}} \eps$ is a net size for our predictions. Combining with the $\approx \frac N T$ loss of the aggregation algorithm over the stated horizon gives $\E[\calK_2^2] \le \eps^2$ and $\eps^2$ excess loss.

\textbf{Tightness of $(\eps, \eps^2)$-recalibration.} Our lower bound is patterned off of a recent work by \cite{CollinaLNR26}, but is much simpler. The hard instance cycles through an \(\eps\)-net inside
\([\frac 1 4, \frac 3 4]\), sets \(q_t\) to the current grid value, and draws \(y_t\sim\Bern(q_t)\). After $O(\eps^{-2})$ total cycles, it is simple to see that outputting the ``truthful predictor'' $p_t = q_t$ at each iteration achieves $\eps$-calibration and $0$ excess loss in $T \approx \eps^{-3}$ iterations. Our lower bound shows that no better strategy exists, even allowing $\eps^2$ excess loss.

The proof first shows in Section~\ref{ssec:rounding} that, up to constant factors in the recalibration guarantee, we may as well assume every $p_t$ lies on the same \(\eps\)-net used to define the instance.
For a rounded value \(v\), let \(\nu_v\) be the number of times \(v\) is predicted, and let $\mu_v=\sum_{t:p_t=v}(q_t-y_t)$ 
be its stochastic fluctuation. We prove a 
bias-miscalibration
decomposition in Lemma~\ref{lem:lower_k1_r}, which shows that calibration error must pay for \(\sum_v|\mu_v|\), except for the amount of
bias the algorithm introduces by moving away from \(q_t\). This latter term is controlled by the excess squared loss bound implied by recalibration. 

The rest of the proof lower bounds the empirical fluctuations \(\sum_v|\mu_v|\). The key challenge is that the binnings in these sums are defined in terms of our own rounded predictions $p_{[T]}$, rather than the original hint sequence $q_{[T]}$. We introduce a charging argument in Lemma~\ref{lem:nuv_upper} that shows the difference is negligible under an excess squared loss bound, concluding the proof.

\subsection{Related work}

To our knowledge, our main problem of recalibration (as stated in Definition~\ref{def:recal}) has only previously been formally studied in \cite{KuleshovE17, OkoroaforKS24}. Here we overview some closely-related literature.

\textbf{Calibeating and calibration.} Several other works \cite{FosterH23, LeeNPR22, HuLSS2025efficient, ChenHJL26} have considered simultaneous calibeating and calibration, as in our Theorem~\ref{thm:calibeat_inf}. Although we summarized their rates in Table~\ref{tab:comparison_calibeat}, here we wish to mention for fair comparison that not all pursued our stated goal of $\eps$-$\ell_1$-calibration and $\eps^2$-calibeating. In particular, some parameterize calibration differently \cite{HuLSS2025efficient, ChenHJL26}, pursue a different tradeoff of $(\eps, \eps)$ \cite{HuLSS2025efficient}, and/or are primarily motivated by the multi-sequence setting \cite{LeeNPR22}. Most directly, \cite{ChenHJL26} obtain simultaneous calibeating and calibration guarantees for the Brier loss, and their conclusion asks whether the best separate binary rates can be achieved simultaneously and whether analogous guarantees extend beyond Brier loss; Theorems~\ref{thm:proper_loss_main} and~\ref{thm:l2_calibeating} answer this question for the binary setting and smooth proper losses considered here. These works proceed using a diverse mix of strategies, notably including swap regret minimization primitives \cite{HuLSS2025efficient, ChenHJL26}.

\textbf{Online calibration.} As discussed in Section~\ref{ssec:results}, online calibration in its own right is a very well-studied problem in the online learning literature, inspired by practical problems like forecasting. Famously,~\cite{FosterV98} proved that $\eps$-$\ell_2$-calibration is achievable in $T = O(\eps^{-3})$ iterations, and that $T = \Omega(\eps^{-2})$ is necessary. Recent breakthroughs, first by \cite{QiaoV21} and subsequently by \cite{DaganDFGKO25}, have improved both bounds by constants in the exponent, although a gap persists. A related line \cite{GuptaJNPR22, GargJRR24, LuoSS25} pursues the stronger goal of $\eps$-\emph{multicalibration} \cite{Hebert-JohnsonK18} in an online setting, for which recent work by \cite{CollinaLNR26} proves a lower bound of $T = \Omega(\eps^{-3})$ iterations. Their hard instance is the basis for our lower bound in Theorem~\ref{thm:lb}.

\textbf{Omniprediction.} We briefly mention a recent line of work on \emph{omniprediction}~\cite{GopalanKRSW22, GopalanHKRW23, HuNRY23, GopalanKR23, GargJRR24, HuTY25, DworkHIPT25, OkoroaforKK25, HuTY26}, a closely-related goal to recalibration as in Definition~\ref{def:recal}. In omniprediction, calibration is used as a means to achieve a simultaneous loss minimization guarantee, against a family of comparators. In our setting, the comparator is a single (or multiple) hint sequence(s), and our goal is to achieve a tight asymmetric tradeoff.

\textbf{Concurrent work.} During the preparation of this manuscript, \cite{FosterH26} published several results on proper calibeating. The most direct overlap is that our Proposition~\ref{prop:online_offline_gap} (combined with Lemma~\ref{lem:lip_implies_smooth}) is essentially the same statement as Theorem 12 in \cite{FosterH26}. All other results in this work remain distinct from \cite{FosterH26}. In particular, for the problem of simultaneous calibeating and calibration considered by Section~\ref{sec:calibeat}, Theorem 9 in \cite{FosterH26} retains the $\approx \eps^{-4}$ convergence rate of the prior work \cite{FosterH23}, whereas our Theorem~\ref{thm:proper_loss_main} obtains the improved $\approx \eps^{-3}$ rate for this problem.

\section{Preliminaries}\label{sec:prelims}

\textbf{Notation.} We denote vectors in boldface and scalars in plaintext. When the ambient dimension is clear from context, $\ve_i$ denotes the $i^{\text{th}}$ standard basis vector. We use $\1_d$ and $\0_d$ to denote the all-ones and all-zeroes vectors in $\R^d$. For $\vv \in \R^d$ and $p \ge 1$, $\norm{\vv}_p$ denotes its $\ell_p$ norm. For $n\in\N$, we let $[n]\defeq\Brace{i\in\N\mid 1\leq i\leq n}$. We abbreviate a scalar sequence $\{p_t\}_{t \in [T]}$ by $p_{[T]}$ for shorthand; similarly, a vector sequence $\{\vv_t\}_{t \in [T]}$ is denoted $\vv_{[T]}$. The $0$-$1$ indicator of an event $\calE$ is denoted $\ind_{\calE}$, and $\Bern(p)$ is the Bernoulli distribution with mean $p$.  

Our algorithms, particularly in Section~\ref{sec:blackwell}, use online regret minimization primitives as subroutines (see e.g., \cite{Shalev-Shwartz07} for background on this topic). In these subroutines, we use the following notation for the \emph{Bregman divergence} in a convex function $\omega: \calX \to \R$, where $\calX \subseteq \R^d$:
\[D_\omega(\vx' \| \vx) \defeq \omega(\vx') - \omega(\vx) - \inprod{\nabla \omega(\vx)}{\vx' - \vx}. \]
This is the residual of a first-order Taylor expansion from $\vx$, and always satisfies $D_\omega(\vx' \| \vx) \ge 0$.

\textbf{Sequential prediction.} 
Throughout, we study \emph{hinted sequential prediction} tasks in binary classification, defined formally in Model~\ref{model:hsp}.
Informally, in applications captured by Model~\ref{model:hsp}, the goal of the learner is to output predictions $p_{[T]}$ that capture the labels $y_{[T]}$ accurately, while improving the hint $q_{[T]}$ in calibration or loss minimization. We note that in Section~\ref{sec:multi} we extend Model~\ref{model:hsp} to a setting where nature provides multiple hints (Model~\ref{model:mhsp}), all of which must be competed against.

In our calibeating applications specifically (cf.\ Section~\ref{sec:calibeat}), we parameterize the hint sequence $q_{[T]}$ by its \emph{sparsity}, as is consistent with the literature. More formally, we say a sequence of scalars $q_{[T]}$ is \emph{$s$-sparse} if the sequence takes on at most $s$ distinct values. We give a discussion on the necessity of this parameterization for our calibeating applications in Remark~\ref{rem:sparsity}.

\textbf{Loss minimization.} Our goal will be to output a prediction sequence $p_{[T]}$ that competes with the input hint sequence $q_{[T]}$ with respect to a \emph{loss} $\ell: [0, 1] \times \{0, 1\} \to \R$ that sends prediction-label pairs to scalar values. When $\ell$ is clear from context, we overload notation and define
\[\ell\Par{p_{[T]}, y_{[T]}} \defeq \frac 1 T \sum_{t \in [T]} \ell(p_t, y_t)\]
to be the average loss over the sequence. We focus on the following special family of losses.

\begin{definition}[Proper loss]\label{def:proper}
We say a loss $\ell: [0, 1] \times \{0, 1\} \to \R$ is \emph{proper} if for all $p \in [0, 1]$,
\[p \in \arg\min_{q \in [0, 1]}\Brace{\E_{y \sim \Bern(p)}\Brack{\ell(q, y)}}.\]
\end{definition}

Many common losses in supervised machine learning satisfy Definition~\ref{def:proper},  
including the \emph{squared loss} and \emph{cross entropy (logistic)} loss, defined as:
\begin{equation}\label{eq:special_loss}
\begin{aligned}
\lsq(p, y) \defeq (p - y)^2,\quad \llog(p, y) \defeq (1 - y)\log\Par{\frac 1 {1 - p}} + y\log\Par{\frac 1 p}.
\end{aligned}
\end{equation}

We also use the following assumption to parameterize our recalibration algorithms.

\begin{assumption}\label{assume:lip}
There exists $L \ge 1$ such that $\ell(\cdot, y)$ is $L$-Lipschitz for all $y \in \{0, 1\}$. 
\end{assumption}
\begin{remark}\label{rem:lip}
Our recalibration algorithm uses Assumption~\ref{assume:lip} in its MLOO construction (Lemma~\ref{lem:recal_mloo}, following \cite{OkoroaforKS24}), and we leave it as an interesting open problem to characterize whether such an assumption is necessary for $(\eps, \eps^2)$-recalibration. Notably, the squared loss $\lsq$ is $L=2$-Lipschitz, but the cross entropy loss $\llog$ is unbounded as $p \to 0$ (when $y = 1$) or $p \to 1$ (when $y = 0$). In practice, this issue could be alleviated by using a Lipschitz approximation to $\llog$.
\end{remark}

\section{Imbalanced Simultaneous Blackwell Approachability}\label{sec:blackwell}

In this section, we give a framework that is an \emph{imbalanced} generalization of the main result of \cite{HuTY26}. That work defined a \emph{simultaneous Blackwell approachability} problem (Problem~\ref{prob:sba}), a multi-set generalization of the classical Blackwell approachibility problem and its application in online learning \cite{Blackwell56, AbernethyBH11}. We recall the relevant problem statement of \cite{HuTY26} and develop some preliminaries in Section~\ref{ssec:setup}. We then give our imbalanced generalization in Section~\ref{ssec:framework}, and derive two basic consequences of it that see direct use in our applications.

\subsection{Setup}\label{ssec:setup}

\textbf{Simultaneous Blackwell approachability.} We state a variant of the main problem studied by \cite{HuTY26} in Problem~\ref{prob:sba}. Here we provide a brief informal description of our variant. 

In this problem, there are $m$ vector-valued functions $\vv^{(i)}$ each with three arguments: an action $\va \in \calA$ chosen by the learner, and a context $q \in \calQ$ and response $y \in \calY$ chosen by nature. There are also $m$ sets $\calV^{(i)}$, the $i^{\text{th}}$ of which is implicitly defined by a ``distinguisher set'' $\calU^{(i)}$ (formally, see \eqref{eq:simul_reg}).  The goal is to choose a sequence $\va_{[T]} \in \calA^T$ such that for any sequences $q_{[T]} \in \calQ^T$, $y_{[T]} \in \calY^T$, each average $\frac 1 T \sum_{t \in [T]} \vv^{(i)}(\va_t, q_t, y_t)$ approximately lies in the $i^{\text{th}}$ set $\calV^{(i)}$. When $m = 1$, this is precisely (a contextual variant of) the classical Blackwell approachabilty as derived in \cite{AbernethyBH11}. This special case is reviewed in Section 3.1, \cite{HuTY26}, see also Problem 2 of that work.

\begin{problem}[Imbalanced simultaneous Blackwell approachability]\label{prob:sba}
Let $a, m \in \N$, let $\calA \subseteq \R^a$ and $\calQ, \calY \subseteq \R$ be compact and convex, and let $\veps \defeq \{\eps^{(i)}\}_{i \in [m]} \in \R^m_{\ge 0}$. For all $i \in [m]$, let $\calU^{(i)} \subseteq \calH^{(i)}$ 
where $\calH^{(i)}$ is a Hilbert space, and define $\calV^{(i)} \defeq \{\vv \in \calH^{(i)} \mid \sup_{\vu \in \calU^{(i)}}\inprod{\vu}{\vv} \le 0 \}$.
Let $\vv^{(i)}: \calA \times \calQ \times \calY \to \calH^{(i)}$ be linear in its first argument, for all $i \in [m]$. 
Our goal is to observe online sequences $q_{[T]} \in \calQ^T$, $y_{[T]} \in \calY^T$ and to choose a possibly randomized sequence $\va_{[T]} \in \calA^T$ so that $\va_t$ depends only on $q_{[t]}$ and $y_{[t - 1]}$, while making, simultaneously for all $i \in [m]$, the expected excess
\begin{equation}\label{eq:simul_reg} \E\Brack{\sup_{\vu\sps i \in \calU^{(i)}} \inprod{\vu\sps i}{\frac 1 T \sum_{t \in [T]} \vv^{(i)}(\va_t, q_t, y_t)}} - \eps^{(i)}\end{equation}
as small as possible, where the expectation is over the learner's randomness.
\end{problem}

We highlight a few distinctions of Problem~\ref{prob:sba} when compared to previous versions. 

First, some works require the $\vv^{(i)}$ to bilinear functions of its inputs, to leverage certain equivalences between satisfiability notions. However, linearity in the second input is not needed in our applications. 
Further, Problem~\ref{prob:sba} as stated in \cite{HuTY26} had all $\eps^{(i)} = \eps$. Our applications use an imbalanced variant: for example, Section~\ref{sec:recal} sets $(\eps^{(1)}, \eps^{(2)}) = (\eps, \eps^2)$ for recalibration (Definition~\ref{def:recal}).

The main challenge in Problem~\ref{prob:sba} is to control all $m$ of the quantities in \eqref{eq:simul_reg} at once, by using a coupled sequence of actions $\va_{[T]}$. Theorem 2 of \cite{HuTY26} gives a sufficient condition for Problem~\ref{prob:sba} to be solvable: each distinguisher set $\calU^{(i)}$ must admit a \emph{no-regret learner}, and the $m$ sets to be approached $\calV^{(i)}$ must jointly satisfy a certain oracle requirement, which we now give.

\begin{definition}[Mixture linear optimization oracle]\label{def:mloo}
In the setting of Problem~\ref{prob:sba}, we call $\oracle$ an $\veps$-\emph{mixture linear optimization oracle} (MLOO) if on inputs $\vw \in \Delta^m$, $\{\vu^{(i)}\}_{i \in [m]} \in \prod_{i \in [m]} \calU^{(i)}$, and $q \in \calQ$, the oracle outputs $\va \in \calA$ satisfying
\begin{equation}\label{eq:mix_halfspace} \sum_{i\in[m]} \vw_i\Par{\inprod{\vu^{(i)}}{\vv^{(i)} (\va, q, y)}- \eps^{(i)}}\le 0 \text{ for all } y \in \calY.\end{equation}
\end{definition}

Definition~\ref{def:mloo} has an intuitive interpretation when specialized to a point mass $\vw = \ve_i$. In this case, the oracle requirement becomes: find $\va$ such that
\begin{equation}\label{eq:one_halfspace}\inprod{\vu^{(i)}}{\vv^{(i)}(\va, q, y)} \le \eps^{(i)} \text{ for all } y \in \calY.\end{equation}
As an example, this condition is attainable (including the exact case $\veps=\0_2$) whenever $\calV^{(i)}$ is ``approachable'' in the sense of Blackwell (see Definition 2, \cite{AbernethyBH11}). This is because \cite{Blackwell56} shows approachability is equivalent to \emph{halfspace satisfiability}, i.e., \eqref{eq:one_halfspace} with $\eps^{(i)} = 0$. This relationship is discussed in detail by Theorem 3 and Lemma 1, \cite{AbernethyBH11}.

Definition~\ref{def:mloo} asks that the inequalities \eqref{eq:one_halfspace} can be satisfied \emph{on average} as specified by any mixture weights $\vw \in \Delta^m$, allowing for slacks given by $\veps$. In general, this is a strictly stronger requirement than asking each $\calV^{(i)}$ to be approachable in isolation (see a counterexample in Lemma 6, \cite{HuTY26}). Consequently, controlling \eqref{eq:simul_reg} can be strictly harder than running $m$ independent approachability procedures, due to the requirement of coupled actions. In Theorem~\ref{thm:sba}, we derive a new sufficient condition for solving Problem~\ref{prob:sba}, specialized to the imbalanced setting.

\textbf{Online learning.} We summarize a standard result from online learning and stochastic convex optimization, for use in proving Theorem~\ref{thm:sba} and our later applications.

\begin{lemma}[Theorem 4.2, \cite{Bubeck15}]\label{lem:mirror}
Fix $T\in\N$ and a set $\xset\subseteq\R^d$ of diameter at most $R$ under a norm $\norm{\cdot}$. Let $r:\xset\to\R$ be $1$-strongly convex in $\norm{\cdot}$,\footnote{This means $D_r(\vx' \| \vx) \ge \half \norm{\vx - \vx'}^2$ for all $\vx, \vx' \in \xset$.}
initialize at $\vx_1 \in \xset$, and for each round $t\in[T]$, let $\vg_t$ be a vector with $\norm{\vg_t}_* \le L$. For any $\eta>0$, define a sequence of iterates by
\[
\vx_{t+1}\gets \arg\min_{\vx\in\xset}\Brace{\inprod{\eta\vg_t-\nabla r(\vx_t)}{\vx}+r(\vx)},\text{ for all } t \in [T].
\]
Then, for all $\vx \in \xset$,
\[
\frac 1 T \sum_{t\in[T]}\inprod{\vg_t}{\vx_t-\vx}
\le
\frac{\eta L^2}{2} + \frac{D_r(\vx \| \vx_1)}{\eta T}.
\]
\end{lemma}

In what follows, we only apply
Lemma~\ref{lem:mirror} in the standard regime of quadratic regularization over a Euclidean ball of radius $R$ (projected gradient descent). These applications use
\begin{equation}\label{eq:quadreg}r(\vx) = \half\norm{\vx}_2^2,\quad D_r(\vx \| \vx') = \half \norm{\vx - \vx'}_2^2,\quad \norm{\cdot} = \norm{\cdot}_* = \norm{\cdot}_2. \end{equation}

\subsection{Main framework}\label{ssec:framework}

We are now ready to state our main framework in Algorithm~\ref{alg:sba}, and analyze it in Theorem~\ref{thm:sba}. The algorithm is displayed in sequence notation for compactness. Operationally, after $\va_t$ is computed, $\vp_t$ is sampled and output before $y_t$ is revealed; the feedback and update lines are then executed after the label is revealed (and are written at the start of the next iteration in the display).

\begin{algorithm2e}\label{alg:sba}
\DontPrintSemicolon
\caption{$\bsa(q_{[T]}, y_{[T]}, \{\alg^{(i)}\}_{i \in [m]}, \oracle, \vw_1, r, \eta)$}
\textbf{Input:} Sequences $q_{[T]} \in \calQ^T$, $y_{[T]} \in \calY^T$, online learners $\{\alg^{(i)}\}_{i \in [m]}$ satisfying \eqref{eq:one_reg_bound}, $\veps$-MLOO $\oracle$ (following notation in Problem~\ref{prob:sba}, Definition~\ref{def:mloo}), initial $\vw_1 \in \Delta^m$, regularizer $r: \Delta^m \to \R$, $\eta > 0$\;
\textbf{Output:}  $\vp_{[T]} \in \calA^T$ such that each $\vp_t$ is output after observing $q_t$ and before observing $y_t$\;
$\vu^{(i)}_1 \gets \alg^{(i)}(\{\})$ for all $i \in [m]$\; \tcp*{Initialize each $\vu^{(i)}_1$ as $\alg^{(i)}$ does before observing any examples.}
$\va_1 \gets \oracle(\vw_1, \{\vu_1^{(i)}\}_{i \in [m]}, q_1)$\;
\For{$2 \le t \le T$}
{
$\vp_{t - 1} \gets $ a random element of $\calA$, sampled using fresh randomness conditionally on $\va_{t - 1}$, such that $\E[\vp_{t - 1} \mid \va_{t - 1}] = \va_{t - 1}$\;\label{line:subsample}
$\vv^{(i)}_{t - 1} \gets \vv^{(i)}(\vp_{t - 1}, q_{t - 1}, y_{t - 1})$ for all $i \in [m]$\;
$\tvg_{t - 1} \gets $ vector in $\R^m$ such that $[\tvg_{t - 1}]_i = \langle \vu_{t - 1}^{(i)}, \vv_{t - 1}^{(i)}\rangle - \eps^{(i)}$ for all $i \in [m]$\;\label{line:gdef}
$\vu^{(i)}_t \gets \alg^{(i)}(\vp_{[t - 1]}, q_{[t - 1]}, y_{[t - 1]})$ for all $i \in [m]$\;
$\vw_t \gets \arg\min_{\vw \in \Delta^m}\{-\inprod{\eta \tvg_{t - 1} + \nabla r(\vw_{t - 1})}{\vw} + r(\vw)\}$\;\label{line:reg}
$\va_t \gets \oracle(\vw_t, \{\vu_t^{(i)}\}_{i \in [m]}, q_t)$\label{line:oracle}
}
$\vp_{T} \gets $ a random element of $\calA$, sampled using fresh randomness conditionally on $\va_T$, such that $\E[\vp_T \mid \va_T] = \va_T$\;
\codeReturn $\vp_{[T]}$
\end{algorithm2e}

\begin{theorem}[Imbalanced simultaneous Blackwell approachability]\label{thm:sba}
In the setting of Problem~\ref{prob:sba}, let $\oracle$ be an $\veps$-MLOO. Also assume that for every $i\in[m]$ and horizon $T\in\N$, there is an online learner $\alg^{(i)}$ which, on inputs $(\va_{[T]}, q_{[T]}, y_{[T]})\in\calA^T\times\calQ^T \times \calY^T$, outputs $\vu_{[T]}^{(i)}$ with each $\vu_t^{(i)}\in\calU^{(i)}$ depending only on $\va_{[t-1]}$, $q_{[t - 1]}$ and $y_{[t-1]}$, and satisfies
    \begin{equation}\label{eq:one_reg_bound}\sup_{\vu^{(i)} \in \calU^{(i)}} \sum_{t \in [T]} \inprod{\vv^{(i)}(\va_t, q_t, y_t)}{\vu^{(i)} - \vu^{(i)}_t} \le \reg^{(i)}(T),\end{equation}
for some $\reg^{(i)}:\N\to\R_{\ge0}$. Finally let $\eta > 0$, let $r:\Delta^m\to\R$ be $1$-strongly convex with respect to a norm $\norm{\cdot}$, and suppose that for some $L,B>0$,
    \begin{equation}\label{eq:width_bound}\Abs{\inprod{\vv^{(i)}(\va, q, y)}{\vu^{(i)}}} \le L,\quad \eps^{(i)} \le L,\end{equation}
for all $i\in[m]$, $(\va,q, y)\in\calA\times\calQ \times \calY$, and $\vu^{(i)}\in\calU^{(i)}$, and
\[
\norm{\Par{\inprod{\vu^{(i)}}{\vv^{(i)}(\va,q,y)}-\eps^{(i)}}_{i\in[m]}}_* \le B
\]
for all $(\va,q,y)\in\calA\times\calQ\times\calY$ and $\vu^{(i)}\in\calU^{(i)}$, where $\norm{\cdot}_*$ is the dual norm. Then, for any sequences $q_{[T]} \in \calQ^T$,  $y_{[T]}\in\calY^T$, Algorithm~\ref{alg:sba} produces $\vp_{[T]}\in\calA^T$ such that $\vp_t$ depends only on $q_{[t]}$, $y_{[t - 1]}$, and
\begin{equation}\label{eq:ith_regret}\E\Brack{\sup_{\vu^{(i)} \in \calU^{(i)}} \inprod{\vu^{(i)}}{\frac 1 T \sum_{t \in [T]} \vv^{(i)}(\vp_t, q_t, y_t)}} \le \eps^{(i)} + \frac{\reg^{(i)}(T)}{T}+ \frac{\eta B^2}{2}+ \frac{D_r(\ve_i\|\vw_1)}{\eta T},\end{equation}
for all $i \in [m]$.
\end{theorem}
\begin{proof}
We adopt the notation of Algorithm~\ref{alg:sba} in this proof. Note that $\va_t$ is chosen on Line~\ref{line:oracle} before $y_t$ is revealed. Conditional on $\va_t$, the iterate $\vp_t$ is sampled using fresh randomness independent of the current history and the realized pair $(q_t,y_t)$. We also let $\vg_t \in \R^m$ satisfy
\[[\vg_t]_i = \inprod{\vu_t^{(i)}}{\vv^{(i)}(\va_t, q_t, y_t)} - \eps^{(i)} \text{ for all } i \in [m],\]
so that $\tvg_t$ is conditionally unbiased for $\vg_t$, given the current history and the realized pair $(q_t,y_t)$, by linearity of each $\vv^{(i)}$ in its first argument.

The first step is to prove that for all $i \in [m]$,
\begin{equation}\label{eq:middle_layer}
\E\Brack{\frac {1} {T} \sum_{t \in [T]} \Par{\inprod{\vu_t^{(i)}}{\vv_t^{(i)}}-\eps^{(i)}}} \le \frac{\eta B^2}{2} + \frac{D_r(\ve_i\|\vw_1)}{\eta T}.
\end{equation}
Indeed, this follows from plugging $\vw \gets \ve_i$ into
\begin{equation}\label{eq:martingale}
\E\Brack{\frac 1 T \sum_{t \in [T]} \inprod{\tvg_t}{\vw}} = \E\Brack{\frac 1 T \sum_{t \in [T]} \inprod{\tvg_t}{\vw_t}} + \E\Brack{\frac 1 T \sum_{t \in [T]} \inprod{\tvg_t}{\vw - \vw_t}} \le  \frac{\eta B^2}{2}  + \frac{D_r(\vw\|\vw_1)}{\eta T}.
\end{equation}
In the only inequality, we bounded each $\E[\inprod{\vw_t}{\tvg_t}] = \E[\inprod{\vw_t}{\vg_t}] \le 0$ using the oracle condition \eqref{eq:mix_halfspace}, which is valid for every $q_t$, $y_t$ defining $\vg_t$. The equality follows from the conditional-unbiasedness statement above. We also applied the regret bound from Lemma~\ref{lem:mirror} to the sequence $-\tvg_{[T]}$, using the assumed dual-norm bound $B$. In particular, the first bound in \eqref{eq:width_bound} holds for any $\va \gets \vp_t \in \calA$.

To conclude, we combine \eqref{eq:middle_layer} with \eqref{eq:one_reg_bound} using $\va_{[T]} \gets \vp_{[T]}$: for all $i\in[m]$,
\begin{align*}
\E\Brack{\sup_{\vu^{(i)} \in \calU^{(i)}} \inprod{\vu^{(i)}}{\frac 1 T \sum_{t \in [T]} \vv_t^{(i)}}} &\le \E\Brack{\frac 1 T \sum_{t \in [T]} \inprod{\vu_t^{(i)}}{\vv_t^{(i)}}} + \E\Brack{\sup_{\vu^{(i)} \in \calU^{(i)}} \frac 1 T \sum_{t \in [T]}\inprod{\vu^{(i)} - \vu_t^{(i)}}{\vv_t^{(i)}}} \\
&\le \eps^{(i)} + \frac{\reg^{(i)}(T)}{T}+ \frac{\eta B^2}{2} + \frac{D_r(\ve_i\|\vw_1)}{\eta T}.
\end{align*}
\end{proof}

We give a convenient corollary of Theorem~\ref{thm:sba}, which applies the quadratic regularizer \eqref{eq:quadreg} to a setting with $m = 2$ approachability sets, and enables tight control for our recalibration applications. 

\begin{corollary}\label{cor:sba_imbalanced_e}
In the setting of Theorem~\ref{thm:sba}, choosing $\vw_1 \gets \ve_2$ and $r = \half\norm{\cdot}_2^2$ in Algorithm~\ref{alg:sba} yields
\begin{align*}
\E\Brack{\sup_{\vu^{(1)} \in \calU^{(1)}} \inprod{\vu^{(1)}}{\frac 1 T \sum_{t \in [T]} \vv^{(1)}(\vp_t, q_t, y_t)}} &\le \eps^{(1)} + \frac{\reg^{(1)}(T) }{T} + 4\eta L^2 + \frac{1}{\eta T}, \\
\E\Brack{\sup_{\vu^{(2)} \in \calU^{(2)}} \inprod{\vu^{(2)}}{\frac 1 T \sum_{t \in [T]} \vv^{(2)}(\vp_t, q_t, y_t)}} &\le \eps^{(2)} + \frac{\reg^{(2)}(T) }{T} + 4\eta L^2.
\end{align*}
\end{corollary}
\begin{proof}
The regularizer $r=\frac12\norm{\cdot}_2^2$ is $1$-strongly convex with respect to $\norm{\cdot}_2$, and the hypotheses of Theorem~\ref{thm:sba} hold with $B=2\sqrt{2}L$, since \eqref{eq:width_bound} implies every coordinate of $\tvg_t$ has magnitude at most $2L$. Also, $D_r(\ve_1 \| \ve_2) = 1$ and $D_r(\ve_2 \| \ve_2) = 0$.
\end{proof}

\section{Recalibration}\label{sec:recal}

In this section, we apply Theorem~\ref{thm:sba} to derive a new algorithm for recalibration of an online predictor $q_{[T]} \in [0, 1]^T$ (Definition~\ref{def:recal}) under a proper loss $\ell$. 
We instate our framework with sets
\begin{equation}\label{eq:ab_def_recal}
\calA \defeq \Delta^\calN,\qquad \calQ = \calY \defeq [0,1],
\end{equation}
where
\begin{equation}\label{eq:net_notation}
\calN\defeq \Brace{\frac i N}_{0 \le i \le N},\quad N \in \N
\end{equation}
is an $\eps$-net of $[0,1]$, for some fixed $\eps \in (0, 1)$. This requires $N=\Omega(1/\eps)$; in the proof of Theorem~\ref{thm:recalibration}, we use the finer choice $N=\Theta(\sqrt{L}/\eps)$ to control both approachability coordinates. We view $\calA$ as actions corresponding to distributions over predictions in $\calN$, and our random sampling in Line~\ref{line:subsample} will simply sample a random strategy $p \in \calN$ according to $\va \in \calA$.

Although the original loss is defined for binary labels, throughout this section we use its affine extension to fractional responses,
\begin{equation}\label{eq:affine_loss_extension}
\ell(p,y) \defeq (1-y)\ell(p,0)+y\ell(p,1),\qquad y\in[0,1].
\end{equation}
Consequently, both vector-payoff functions below are affine in $y$. It therefore suffices to verify the MLOO inequalities at the endpoints $y\in\{0,1\}$; the inequalities then hold throughout the convex response set $\calY=[0,1]$.

Moreover, we define approachability sets
\begin{equation}\label{eq:uv1def_recal}
\calU^{(1)}\defeq [-1,1]^{\calN},\qquad
\vv^{(1)}(\va,q, y) \defeq \{\va_s(s-y)\}_{s\in\calN},
\end{equation}
living in $\calH^{(1)} \defeq \R^{\calN}$, and
\begin{equation}\label{eq:uv2def_recal}
\calU^{(2)}\defeq \{1\},\qquad
v^{(2)}(\va,q, y) \defeq \sum_{s \in \calN} \va_s (\ell(s, y) - \ell(q, y)),
\end{equation}
living in $\calH^{(2)} \defeq \R$. We pause to make a few observations about this setup. 

First, the role of $\vv^{(1)}$ is to enforce calibration, for which the hint sequence in Definition~\ref{def:recal} is irrelevant, so it is independent of $q$. Moreover, $\norm{\vv}_1 = \sup_{\vu \in \calU^{(1)}} \inprod{\vu}{\vv}$, so the left-hand side of \eqref{eq:ith_regret} with $i \gets 1$ is $\E[\calK_1(p_{[T]}, y_{[T]})]$, where $p_t \in \calN$ is the prediction associated with a deterministic action (point mass) $\vp_t \in \calA$. Thus, the first approachability set \eqref{eq:uv1def_recal} exactly captures calibration error. 

Analogously, when all actions $\va$ taken by the algorithm are point masses, it is straightforward to see that the second approachability set \eqref{eq:uv2def_recal} captures the excess loss aspect of Definition~\ref{def:recal}.

To instantiate Theorem~\ref{thm:sba} in the setting of \eqref{eq:ab_def_recal}, \eqref{eq:uv1def_recal}, \eqref{eq:uv2def_recal}, we require three ingredients: no-regret learners over $\calU^{(1)}$ and $\calU^{(2)}$ in the sense of \eqref{eq:one_reg_bound}, and an MLOO for the instance. In Section~\ref{ssec:mloo}, we repurpose a result from \cite{OkoroaforKS24} for use as an MLOO with appropriate parameters. In Section~\ref{ssec:recal}, we put together the components to prove Theorem~\ref{thm:recalibration}, our main recalibration result.

\subsection{Imbalanced MLOO implementation}\label{ssec:mloo}

In this section, we give a short rederivation of Lemma 7, \cite{OkoroaforKS24}. The subroutine in that result is used within Theorem~\ref{thm:sba} as our MLOO with parameters $(\eps^{(1)}, \eps^{(2)}) \approx (\eps, \eps^2)$, in the setting of \eqref{eq:ab_def_recal}, \eqref{eq:uv1def_recal}, \eqref{eq:uv2def_recal}. Recall from Definition~\ref{def:mloo} that the input to the MLOO (when $m = 2$) is mixture weights $\vw = (\alpha, \beta)$ satisfying $\alpha + \beta = 1$, $\vu \in \calU^{(1)}$, and $q \in \calQ$, and the output is $\va \in \calA$ satisfying the following inequalities at the two endpoints (and hence, by affinity, for every $y\in\calY$):
\[\alpha\inprod{\vu}{\vv^{(1)}(\va, q, y)} + \beta v^{(2)}(\va, q, y)\le \alpha\eps^{(1)} + \beta\eps^{(2)}, \text{ for all } y \in \{0, 1\}.\]

We begin by recalling some basic facts about proper losses.

\begin{lemma}\label{lem:monotonicity}
Let $\ell: [0, 1] \times \{0, 1\} \to \R$ be a proper loss. Then, $\ell(\cdot, 0)$ is monotone nondecreasing and $\ell(\cdot, 1)$ is monotone nonincreasing (both in their first argument).
\end{lemma}
\begin{proof}
Let $0 \le p < q \le 1$, and define $\Delta_0 \defeq \ell(q, 0) - \ell(p, 0)$ and $\Delta_1 \defeq \ell(q, 1) - \ell(p, 1)$. Properness of $\ell$ implies $(1 - p)\ell(p, 0) + p\ell(p, 1) \le (1 - p)\ell(q, 0) + p\ell(q, 1)$ (and likewise with $p, q$ swapped), so
\[(1 - p)\Delta_0 + p\Delta_1 \ge 0 \ge (1 - q)\Delta_0 + q\Delta_1.\]
Rearranging yields
\[(q-p)(\Delta_0 - \Delta_1) \ge 0 \implies \Delta_1 \le \Delta_0.\]
Finally, we have the desired
\[\Delta_0 \ge (1 - p)\Delta_0 + p\Delta_1 \ge 0,\quad \Delta_1 \le (1 - q)\Delta_0 + q\Delta_1 \le 0.\]
\end{proof}

We now prove our main helper result in the MLOO implementation.

\begin{lemma}[Lemma 7, \cite{OkoroaforKS24}]\label{lem:oracle_helper}
    Let $N\in\N$ be even, and let $\vu \in \R^{\{0\} \cup [N]}$ satisfy $\norm{\vu}_\infty \le 1$.
    Let $c,d:[0,1]\times\{0,1\}\rightarrow \mathbb{R}$, and define a function $f:\left(\{0\}\cup[N]\right)\times \{0,1\}\rightarrow \mathbb{R}$ by
    \[
    f(i,y)\defeq \vu_ic\Par{\frac{i}{N},y}+d\Par{\frac{i}{N},y}.
    \]
    If the following assumptions hold:
    \begin{enumerate}
        \item $c(p,y)=p - y$,
        \item $d$ is proper, $L$-Lipschitz, $d(\cdot,0)$ is monotone nondecreasing and $d(\cdot,1)$ is monotone nonincreasing (both in the first argument), and there exists a $p \in [0, 1]$ such that $d(p,0) = d(p,1) =0$,
        \item $f(0,0)\leq 0$ and $f(N,1)\leq 0$,
    \end{enumerate}
    then there exists a distribution $\mathcal{D}$ over $\{0\}\cup[N]$ such that
    \begin{equation}\label{eq:aug_quadrant}\E_{i\sim \mathcal{D}}\Brack{\begin{pmatrix}f(i,0) \\ f(i, 1) \end{pmatrix}}\in\Brace{\left(-\infty, \frac{1}{N}+\frac{4L}{N^2}\right]\times\left(-\infty, \frac{1}{N}+\frac{4L}{N^2}\right]}.
    \end{equation}
    Moreover, $\calD$ is either a point mass, or supported on two neighboring indices $\{j, j+1\}$.
    \begin{figure}[ht!]
\centering
\includegraphics[width=0.5\textwidth]{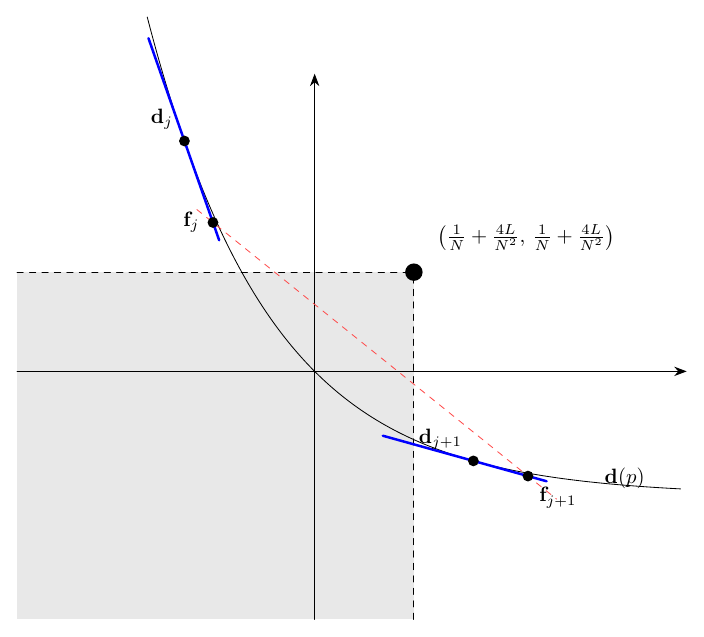}
\end{figure}
\end{lemma}
\begin{proof}
For convenience in this proof, for $p \in [0, 1]$, we define vector-valued functions
\[\vc(p) \defeq \begin{pmatrix}c\Par{p, 0} \\ c\Par{p, 1} \end{pmatrix},\quad \vd(p) \defeq \begin{pmatrix} d\Par{p, 0} \\ d\Par{p, 1} \end{pmatrix},\]
and denote for shorthand $\vc_i \defeq \vc(\frac i N)$, $\vd_i \defeq \vd(\frac i N)$, and $\vf_i \defeq \vu_i \vc_i + \vd_i$ for all $0 \le i \le N$, so that
\[\vf_i = \begin{pmatrix} f(i, 0) \\ f(i, 1) \end{pmatrix}.\]
Our goal is to show that either some $\vf_i$ lies in the set on the right-hand side of \eqref{eq:aug_quadrant}, or the line segment between two adjacent $\vf_j$, $\vf_{j + 1}$ lies in that set. We also define the orthogonal direction
\[\vc_i^\perp \defeq \begin{pmatrix} 1 - \frac i N \\ \frac i N \end{pmatrix} \text{ for all } i \in \{0\} \cup [N].\]
Next, for all $i \in \{0\} \cup [N]$, let $\sfL_i$ be the line through $\vd_i$ and $\vf_i$: since $\inprod{\vc_i}{\vc_i^\perp} = 0$,
    \begin{align*}      
   \sfL_i \defeq \{\vd_i+\lambda \vc_i\mid \lambda\in \R\} = \Brace{\vv \in \R^2 \mid \left\langle\vc_i^\perp,\vv - \vd_i\right\rangle= 0}.
   \end{align*}
   We first show $\sfL_i$ is a supporting line of $\vd(p)$. Since $d$ is proper, for any point $\vx = \vd(p)$,
   \begin{equation}\label{eq:proper_d}
   \Par{1 - \frac i N} d\Par{p, 0} + \frac i N d\Par{p, 1} = \left\langle\vc_i^\perp,\vx\right\rangle\geq\left\langle \vc_i^\perp,\vd_i\right\rangle = \Par{1 - \frac i N}d\Par{\frac i N, 0} + \frac i N d\Par{\frac i N, 1}.
   \end{equation}
   Indeed, a similar argument shows that there is a supporting line of $\vd(p)$ at any $p \in (0, 1]$ with slope $\frac{c(p, 1)}{c(p, 0)} = 1 - \frac 1 p$, which is increasing in $p$; the endpoint $p=0$ is obtained by the limiting vertical supporting line. Thus, the region of $\R^2$ lying above $\vd(p)$ is convex.

   Now, if any $\vf_i$ lies in the third quadrant $\R^2_{\le 0}$, \eqref{eq:aug_quadrant} holds for a point mass $\calD$. Assume henceforth there is no such $\vf_i$. We claim that no $\vf_i$ is strictly in the first quadrant either; if there were such a $\vf_i$, then it cannot be $i = 0$ or $i = N$ by assumption, so $\inprod{\vc_i^\perp}{\vf_i} > 0$. However,
   \begin{equation}\label{eq:line_neg}\inprod{\vc_i^\perp}{\vf_i} = \inprod{\vc_i^\perp}{\vd_i} \le 0,\end{equation}
   where we used the definition of $\sfL_i$ and the inequality applied \eqref{eq:proper_d} with the $p$ such that $\vd(p) = \0_2$.
   
   Thus have that every $\vf_i$ lies either in the second or fourth quadrants. By assumption, $\vf_0$ must lie in the second quadrant and $\vf_N$ must lie in the fourth quadrant. There then exists a $j \in \{0,\ldots,N-1\}$ such that $\vf_j$ is in the second quadrant and $\vf_{j+1}$ is in the fourth quadrant. Since $N$ is even, after swapping labels $y=0$ and $y=1$ if necessary, we may assume $j\ge \frac N 2$.
    
    \begin{figure}[ht!]
\centering
\includegraphics[width=0.5\textwidth]{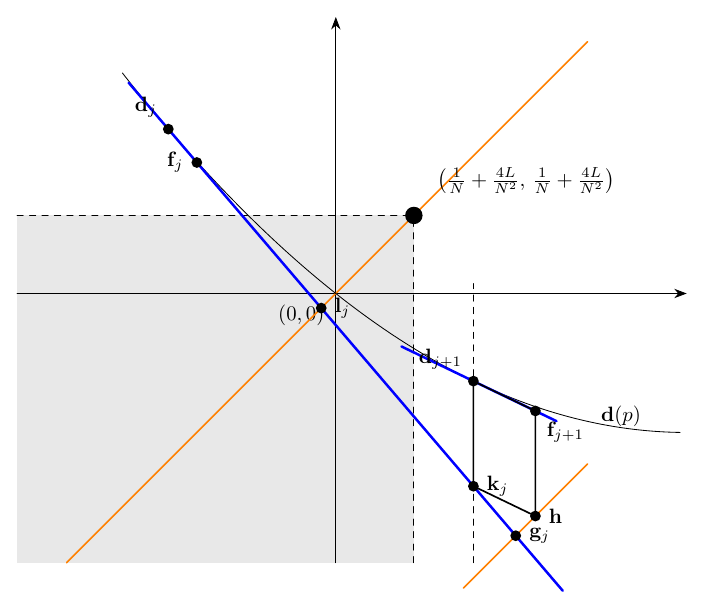}
\end{figure}

    Let $\vk_j$ be the intersection of the vertical line through $\vd_{j+1}$ with $\sfL_j$. Since the slope of $\sfL_j$ is $1-\frac{N}{j}$ and the slope of $\sfL_{j+1}$ is $1-\frac{N}{j+1}$, the point $\vk_j$ lies below $\vd_{j+1}$. Because $\vd_j$ lies above $\sfL_{j + 1}$, a direct calculation shows that the vertical gap between $\vd_{j + 1}$ and $\vk_j$ is at most the product of the horizontal gap between $\vd_j, \vd_{j + 1}$, and the difference in slopes between $\sfL_j$, $\sfL_{j + 1}$. This is
    \begin{equation}\label{eq:gap_bound}\frac{N}{j(j+1)}\cdot\Par{d\Par{\frac{j + 1}{N}, 0} -d\Par{\frac{j}{N}, 0}}\leq \frac{N}{(N/2)(N/2 + 1)} \cdot \frac{L}{N}\le \frac{4L}{N^2},\end{equation}
    where we applied Lipschitzness of $d$. 
    
    Finally, we are ready to define $\calD$. Let $\vi_j$ be the intersection between $\sfL_j$ and the line through the origin with slope $1$. We claim that $\vi_j$ lies between $\vf_j$ and the following point on $\sfL_j$: 
    \[
    \vg_j \defeq \vk_j+ \vu_{j+1}\vc_j = \vk_j +(\vf_{j+1}-\vd_{j+1})+\vu_{j+1}(\vc_j-\vc_{j+1}).
    \]
   This is because $\vf_j$'s second coordinate is larger than its first. On the other hand, 
   \[\vg_j = \vf_{j + 1} + \Par{\vk_j - \vd_{j + 1}} - \frac{\vu_{j + 1}}{N} \1_2 ,\]
   so its first coordinate is larger than its second, because $\vf_{j + 1}$ has this property and $\vk_j - \vd_{j + 1}$ is vertically downwards. Thus, $\vi_j = (1 - \lam) \vf_j + \lam \vg_j$ for some $\lam \in [0, 1]$.

    We output $\vf_j$ with probability $1 - \lam$, and $\vf_{j + 1}$ with probability $\lam$. We claim that $(1 - \lam)\vf_j + \lam \vf_{j + 1}$ lies in the set \eqref{eq:aug_quadrant}, which would conclude the proof. To see this, note that $\vi_j$ is in the third quadrant, because its two coordinates are equal (say, to $t$), so combining \eqref{eq:line_neg} with $\vi_j \in \sfL_j$,
    \[t = \Par{1 - \frac j N} t + \frac j N \cdot t  = \inprod{\vc^\perp_j}{\vd_j} \le 0.\]
    Finally, we can bound the coordinates of $(1 - \lam)\vf_j + \lam\vf_{j + 1}$:
    \begin{align*}
    (1 - \lam)\vf_j + \lam\vf_{j + 1} = \vi_j + \lam\Par{\Par{\vd_{j + 1} - \vk_j} + \frac{\vu_{j + 1}}{N}\1_2}.
    \end{align*}
    Indeed, both coordinates of $\lam(\vd_{j + 1} - \vk_j)$ are at most $\frac{4L}{N^2}$ using \eqref{eq:gap_bound}, and both coordinates of $\frac{\lam\vu_{j + 1}}{N}\1_2$ are at most $\frac 1 N$ using the assumption on $\vu$. Since $\vi_j \in \R_{\le 0}^2$, we are done.
\end{proof}

\begin{lemma}\label{lem:recal_mloo}
Suppose $N$ is even. Under Assumption~\ref{assume:lip}, there is a $\veps$-MLOO for \eqref{eq:ab_def_recal}, \eqref{eq:uv1def_recal}, \eqref{eq:uv2def_recal}, with
\[\eps^{(1)} = \frac 1 N,\quad \eps^{(2)} = \frac{4L}{N^2}.\]

Moreover, the output is a $2$-sparse element of $\Delta^{\calN}$ supported on neighboring predictions $\{\frac j N, \frac {j + 1} N\}$.
\end{lemma}
\begin{proof}
Consider inputs to the MLOO, $\vw = (\alpha, \beta) \in \Delta^2$, $\vu \in \calU^{(1)}$, and $q \in \calQ$. We first handle the boundary case $\alpha = 0$. For every $\alpha_k > 0$ with $\alpha_k \to 0$, set $\beta_k = 1-\alpha_k$ and apply the construction below to $(\alpha_k,\beta_k)$. Since there are only finitely many neighboring supports in $\calN$, a subsequence of the resulting distributions has a common neighboring support and converges to some $\va \in \Delta^{\calN}$. Passing to the limit in the MLOO inequalities gives $v^{(2)}(\va,q,y)\le 4L/N^2$ for both $y\in\{0,1\}$, and hence for every $y\in\calY$ by affinity. This is exactly the desired guarantee when $\alpha=0$.

Henceforth assume $\alpha>0$. Define
\[d\Par{p, y} \defeq \frac \beta \alpha \Par{\ell(p, y) - \ell(q, y)}.\]
We check that this function satisfies the assumptions of Lemma~\ref{lem:oracle_helper}. Indeed, it is proper and $\frac{\beta L}{\alpha}$-Lipschitz by assumption, satisfies the monotonicity conditions by Lemma~\ref{lem:monotonicity}, and satisfies $d(q, 0) = d(q, 1) = 0$. Moreover, propriety implies $f(0, 0) = d(0, 0) \le 0$ and $f(N, 1) = d(1, 1) \le 0$.

Now let $\va$ be the natural identification of the output distribution from Lemma~\ref{lem:oracle_helper} with an element of $\Delta^{\calN}$. By the conclusion of Lemma~\ref{lem:oracle_helper}, we have for either endpoint $y \in \{0,1\}$,
\[\inprod{\vu}{\vv^{(1)}(\va, q, y)} + \frac{\beta}{\alpha}v^{(2)}(\va, q, y) \le \frac 1 N + \frac{4\beta L}{\alpha N^2}.\]
Rescaling this guarantee by $\alpha$ gives the desired MLOO inequality at both endpoints. The affine extension \eqref{eq:affine_loss_extension} then gives the inequality for every $y\in\calY$, as required by Definition~\ref{def:mloo}.
\end{proof}

\subsection{Main result}\label{ssec:recal}

We can now combine our MLOO implementation in Lemma~\ref{lem:recal_mloo} with standard regret minimization bounds within Theorem~\ref{thm:sba} to obtain our main result on recalibration. We briefly discuss the required online learners $\alg^{(1)}$ and $\alg^{(2)}$ satisfying \eqref{eq:one_reg_bound}. First, because $\calU^{(2)}$ is a singleton, $\alg^{(2)}$ can trivially achieve $\reg^{(2)}(T) = 0$ by returning $1$ every iteration. We provide $\alg\sps 1$ in the next result.

\begin{lemma}\label{lem:recal_regret}
    There is an online learning algorithm $\alg\sps 1$ that satisfies \eqref{eq:one_reg_bound} with
    \[\reg^{(1)}(T) \defeq \frac{\eps T}{8} + \frac {16N} {\eps}.\]
\end{lemma}
\begin{proof}
This follows from projected gradient descent (Lemma~\ref{lem:mirror} with the settings in \eqref{eq:quadreg}), $\vx_1 \gets \0_{\calN}$, and $\eta \gets \frac \eps {16}$. More precisely, we use that each $\vu \in \calU^{(1)}$ has
\[\norm{\vu}_2 \le \sqrt{N + 1} \norm{\vu}_\infty \le \sqrt{2N} =: R, \]
and each feedback vector $\vv \defeq \vv^{(1)}(\va, q, y)$ for $(\va, q, y) \in \calA \times \calQ \times \calY$ has
    $\norm{\vv}_2\leq\norm{\vv}_1\leq 2 =: L$.
\end{proof}

\begin{theorem}\label{thm:recalibration}
Let $\ell: [0, 1] \times \{0, 1\} \to \R$ be proper such that Assumption~\ref{assume:lip} holds, and let $\eps\in(0,\half)$. In the setting of Model~\ref{model:hsp}, there is an algorithm that outputs $p_{[T]} \in [0, 1]^T$ that is $(\eps, \eps^2)$-recalibrated against $(q_{[T]}, y_{[T]})$ under $\ell$, if for an appropriate constant, 
\[T = \Omega\Par{\frac{L^2}{\eps^3}}.\]
\end{theorem}
\begin{proof}
We apply Corollary~\ref{cor:sba_imbalanced_e} with the setup in \eqref{eq:ab_def_recal}, \eqref{eq:uv1def_recal}, \eqref{eq:uv2def_recal}, where we use Lemma~\ref{lem:recal_mloo} with the even choice $N = 2\lceil \frac{2\sqrt L}{\eps}\rceil$ as the MLOO, Lemma~\ref{lem:recal_regret} as $\alg^{(1)}$, and the trivial learner over $\calU^{(2)}$ as $\alg^{(2)}$. We now bound each of the quantities in Theorem~\ref{thm:sba} and explain its implications to recalibration.

First, the MLOO parameters yielded by Lemma~\ref{lem:recal_mloo} with the stated value of $N$ are
\begin{equation}\label{eq:veps_def}\eps^{(1)} = \frac 1 N \le \frac \eps 4,\quad \eps^{(2)} = \frac {4L} {N^2} \le \frac{\eps^2}{4}.\end{equation}
This implies we can take $L$ as given in the theorem statement in \eqref{eq:width_bound}, because $|\inprod{\vu}{\vv}| \le 1$ for $\vu \in \calU^{(1)}$ and $\vv = \vv^{(1)}(\va, q, y)$ by the $\ell_1$-$\ell_\infty$ H\"older's inequality, and $|\vv^{(2)}(\va, q, y)| \le L$ by Lipschitzness of $\ell$. Next, Lemma~\ref{lem:recal_regret} gives
\[\frac{\reg^{(1)}(T)}{T} \le \frac \eps 8 + \frac{16N}{\eps T}\]
and we clearly have $\reg^{(2)}(T) = 0$. Therefore, if we take
\[\eta \gets \frac{\eps^2}{16L^2},\]
and let $T$ be a sufficiently large constant multiple of $L^2/\eps^3$, then the two bounds in Corollary~\ref{cor:sba_imbalanced_e} are at most $(\eps, \eps^2)$ respectively. Here we use $L\ge 1$, so $N=O(\sqrt L/\eps)$ and the term $16N/(\eps T)$ is dominated by the stated horizon. We therefore sample $\vp_t$ as a point mass according to $\va_t$, using fresh randomness on every round; its conditional expectation is $\va_t$ even after conditioning on the current history and any fixed realized pair $(q_t,y_t)$. We output the prediction $p_t$ associated with this point mass.

By the definitions in \eqref{eq:uv1def_recal}, \eqref{eq:uv2def_recal}, when every prediction output is a point mass, Corollary~\ref{cor:sba_imbalanced_e} implies $(\eps, \eps^2)$-recalibration: the first inequality is exactly
\[\E\Brack{\frac 1 T \sum_{s \in \calN} \Abs{n_T(s) (s - \by_T(s))}} = \E\Brack{\calK_1\Par{p_{[T]}, y_{[T]}}}\le \eps,\]
and the second inequality is exactly
\[\E\Brack{\frac 1 T \sum_{t \in [T]} \ell(p_t, y_t) - \ell(q_t, y_t)} = \E\Brack{\ell\Par{p_{[T]}, y_{[T]}} - \ell\Par{q_{[T]}, y_{[T]}}} \le \eps^2.\]
\end{proof}

As shown in Theorem~\ref{thm:lb}, Theorem~\ref{thm:recalibration} is minimax optimal up to constant factors.

\subsection{High-probability recalibration}

Theorem~\ref{thm:recalibration} proves that taking $T \approx \eps^{-3}$ suffices for the expected recalibration guarantee in Definition~\ref{def:recal} for Lipschitz proper losses. Here we improve Theorem~\ref{thm:recalibration} to give a high-probability recalibration bound. Our applications in Sections~\ref{sec:calibeat} and~\ref{sec:multi} are fully deterministic except in their reliance on Theorem~\ref{thm:recalibration}, so they also straightforwardly extend to hold with high probability.

To obtain our high-probability result, we separate the prediction output to nature from the feedback used by the approachability procedure. On round $t$, after nature has fixed $(q_t,y_t)$ and revealed $q_t$, the procedure chooses a mixed action $\va_t$ and samples and outputs $p_t\sim\va_t$ before $y_t$ is revealed. Thus, future hint-label pairs may depend on the realized predictions $p_{[t]}$, as allowed by Model~\ref{model:hsp}. After $y_t$ is revealed, however, the internal approachability update uses the deterministic mixed action $\va_t$ rather than its sampled point mass. Using the same action, context, and approachability sets \eqref{eq:ab_def_recal}, \eqref{eq:uv1def_recal}, \eqref{eq:uv2def_recal} as before gives the following pathwise guarantee for the mixed actions.

\begin{lemma}\label{lem:det_sba}
Let $\ell: [0, 1] \times \{0, 1\} \to \R$ be proper such that Assumption~\ref{assume:lip} holds, and let $\eps \in (0, \half)$. In the setting of Model~\ref{model:hsp}, there is an even $N=\Theta(\sqrt L/\eps)$ and an algorithm that outputs $\va_{[T]} \in (\Delta^{\calN})^T$ (where each $\va_t$ has the same measurability assumptions as each $p_t$ in Model~\ref{model:hsp}), such that the following conditions simultaneously hold deterministically, if for an appropriate constant, $T = \Omega(\frac{L^2}{\eps^3})$:
\begin{equation}\label{eq:det_sba_conds}
\norm{\vc\Par{\va_{[T]}, y_{[T]}}}_1 \le \frac \eps 2,\quad \frac 1 T \sum_{t \in [T]} \Par{\ell\Par{\va_t, y_t} - \ell(q_t, y_t)} \le \frac{\eps^2} 2,
\end{equation}
where we overload the definition $\ell(\va, y) \defeq \sum_{s \in \calN} \va_s \ell(s, y)$ for any $\va \in \Delta^{\calN}$, and we define the miscalibration vector $\vc(\va_{[T]}, y_{[T]}) \in \R^{\calN}$ by
\begin{equation}\label{eq:miscal_def}\Brack{\vc(\va_{[T]}, y_{[T]})}_{s} = \frac 1 T \sum_{t \in [T]} \Brack{\va_t}_s \Par{s - y_t},\text{ for all } s \in \calN.\end{equation}
Moreover, each $\va_t$ is deterministically supported on at most two strategies $s, s'$, with $|s - s'| \le \frac \eps 2$.
\end{lemma}
\begin{proof}
Run the proof of Corollary~\ref{cor:sba_imbalanced_e} with $\vp_t=\va_t$ in every internal update. The conditional-unbiasedness step in that proof then becomes a pointwise equality, so its guarantees hold pathwise for every realized adaptive sequence. Combining this observation with our approachability sets \eqref{eq:uv1def_recal}, \eqref{eq:uv2def_recal}, using the same parameters as in Theorem~\ref{thm:recalibration} up to constant factors, gives \eqref{eq:det_sba_conds}. The condition on each $\va_t$ results from our MLOO construction in Lemma~\ref{lem:recal_mloo}, using that $\frac 1 N \le \frac \eps 2$.
\end{proof}

We next bound the error from sampling each mixed action $\va_t$ into the online prediction $p_t \in \calN$, on both terms in \eqref{eq:det_sba_conds}. The bounds are stated conditionally so that the mixed actions, labels, and future play may all depend on previous sampled predictions.

\begin{lemma}\label{lem:cal_subsample}
Let $\delta \in (0, \half)$, and let $\{\calG_t\}_{t\in[T]}$ describe the history immediately before $p_t$ is sampled, after the current pair $(q_t,y_t)$ has been fixed. Suppose $\va_t\in\Delta^{\calN}$ and $y_t\in\{0,1\}$ are $\calG_t$-measurable, and conditionally sample $p_t\sim\va_t$, so that $\E[\vp_t\mid\calG_t]=\va_t$ for $\vp_t\defeq\ve_{p_t}$. Then if for an appropriate constant,
\[T = \Omega\Par{\frac{|\calN|}{\eps^2}\log\Par{\frac{4|\calN|}{\delta}}},\]
we have with probability $\ge 1 - \frac \delta 2$, 
\[\norm{\vc\Par{\va_{[T]}, y_{[T]}} - \vc\Par{\vp_{[T]}, y_{[T]}}}_1 \le \frac \eps 2.\]
\end{lemma}
\begin{proof}
For each $t\in[T]$, define a vector $\vx_t\in\R^{\calN}$ by
\[[\vx_t]_s\defeq\Par{\ind\Par{p_t=s}-[\va_t]_s}(s-y_t),\qquad s\in\calN.\]
For every sign vector $\boldsymbol{\xi}\in\{-1,1\}^{\calN}$, the scalar random variables $Z_t(\boldsymbol{\xi})\defeq\inprod{\boldsymbol{\xi}}{\vx_t}$ form a martingale difference sequence with respect to the sampling history. Indeed, conditional on $\calG_t$,
\[
Z_t(\boldsymbol{\xi})
=\xi_{p_t}(p_t-y_t)-\sum_{s\in\calN}[\va_t]_s\xi_s(s-y_t),
\]
which has conditional mean zero and absolute value at most $2$. Therefore, Azuma--Hoeffding's inequality and a union bound over the $2^{|\calN|}$ sign vectors imply that, with probability at least $1-\delta/2$,
\[
\norm{\sum_{t\in[T]}\vx_t}_1
=\max_{\boldsymbol{\xi}\in\{-1,1\}^{\calN}}\sum_{t\in[T]}Z_t(\boldsymbol{\xi})
\le \sqrt{8T\Par{|\calN|\log 2+\log\Par{\frac{2}{\delta}}}}.
\]
Since
\[
\norm{\vc\Par{\va_{[T]},y_{[T]}}-\vc\Par{\vp_{[T]},y_{[T]}}}_1
=\frac1T\norm{\sum_{t\in[T]}\vx_t}_1,
\]
the conclusion follows from the stated lower bound on $T$.
\end{proof}

\begin{lemma}\label{lem:loss_subsample}
Let $\delta \in (0, \half)$, and suppose the conditional sampling setup of Lemma~\ref{lem:cal_subsample} holds. Also, suppose that Assumption~\ref{assume:lip} holds, and that each $\va_t$ is supported on at most two strategies $s, s'$, with $|s - s'| \le \frac \eps 2$. Then if for an appropriate constant,
\[T = \Omega\Par{\frac{L^2\log\Par{\frac 1 {\delta\eps}}}{\eps^2}},\]
we have with probability $\ge 1 - \frac \delta 2$, 
\[\frac 1 T \sum_{t \in [T]} \Par{\ell\Par{p_t, y_t} - \ell\Par{\va_t, y_t}} \le \frac {\eps^2} 2,\]
where $\ell(\va_t,y_t)\defeq \sum_{s\in\calN}[\va_t]_s\ell(s,y_t)$.
\end{lemma}
\begin{proof}
For all $t \in [T]$, define a random variable 
\[Y_t \defeq \ell(p_t, y_t) - \ell(\va_t, y_t).\]
Then $\E[Y_t\mid\calG_t]=0$, so $Y_t$ is a martingale difference sequence satisfying
\[\Abs{Y_t} \le \frac{L\eps}{2} \text{ with probability } 1.\]
Here, we used Lipschitzness of $\ell$, and the support assumption on $\va_t$. Thus, for $T$ taken as stated, the Azuma-Hoeffding inequality shows the desired claim holds with probability $\ge 1 - \frac \delta 2$:
\[\frac 1 T \sum_{t \in [T]} Y_t \le \frac {\eps^2} 2.\]
\end{proof}

We conclude with our high-probability recalibration result, by combining Lemmas~\ref{lem:det_sba},~\ref{lem:cal_subsample}, and~\ref{lem:loss_subsample}.

\begin{corollary}\label{cor:recal_whp}
Let $\ell: [0, 1] \times \{0, 1\} \to \R$ be proper such that Assumption~\ref{assume:lip} holds, and let $(\eps, \delta) \in (0, \half)^2$. In the setting of Model~\ref{model:hsp}, there is an algorithm that outputs $p_{[T]} \in [0, 1]^T$ such that the following conditions hold with probability $\ge 1 - \delta$:
\begin{align*}
\calK_1\Par{p_{[T]}, y_{[T]}} \le \eps,\quad \ell\Par{p_{[T]}, y_{[T]}} - \ell\Par{q_{[T]}, y_{[T]}} \le \eps^2,
\end{align*}
if for an appropriate constant,
\[T = \Omega\Par{\frac{L^2}{\eps^3} + \frac{\sqrt L\log\Par{\frac{\sqrt L}{\delta\eps}}}{\eps^3} + \frac{L^2\log\Par{\frac 1 {\delta\eps}}}{\eps^2}}.\]
\end{corollary}
\begin{proof}
Run the mixed-action procedure from Lemma~\ref{lem:det_sba} with the even choice $N=\Theta(\sqrt L/\eps)$ specified there. On each round, conditionally sample and output $p_t\sim\va_t$ before $y_t$ is revealed, while using $\va_t$ itself in the internal approachability update after the reveal. The pathwise guarantees \eqref{eq:det_sba_conds} therefore hold for the realized adaptive sequence. Since $|\calN|=N+1=O(\sqrt L/\eps)$, the three terms in the stated horizon respectively imply the conditions of Lemmas~\ref{lem:det_sba}, \ref{lem:cal_subsample}, and \ref{lem:loss_subsample}. The two sampling conclusions hold simultaneously with probability at least $1-\delta$ by a union bound, and combining them with \eqref{eq:det_sba_conds} gives the claim.
\end{proof}

\section{Calibeating and Calibration}\label{sec:calibeat}

In this section, we apply Theorems~\ref{thm:recalibration} and~\ref{thm:l2recalibration} to improve existing calibeating algorithms, to simultaneously ensure calibration (Definition~\ref{def:calerr}) and $\eps^2$-calibeating (Definition~\ref{def:calibeating}) in $T \approx \eps^{-3}$ iterations. Previously, these guarantees were known in isolation \cite{FosterV98, FosterH23}, but not simultaneously.

While natural in its own right, this problem is well-motivated from the perspective of the \emph{refinement score} \eqref{eq:refinement}, the loss benchmark used in the literature on calibeating \cite{LeeNPR22, FosterH23, ChenHJL26}. Recall that the refinement score is the loss achieved by using the hindsight average-label predictor $\by_T(v)$, in all iterations where $q_t = v$. This hindsight predictor is infeasible to implement without knowledge of $y_{[T]}$, so calibeating aims to compete with its loss in an online setting. 

Beyond achieving a strong loss guarantee, however, the hindsight predictor $v \to \by_T(v)$ is clearly calibrated. On the other hand, existing calibeating algorithms do not ensure calibration at the same rate as \cite{FosterV98}; requiring this additional property has thus far resulted in worse asymptotics. The conclusion of \cite{ChenHJL26} asks whether this loss is necessary, i.e., whether one can match the best separate rates for calibeating and calibration simultaneously for the Brier loss, and whether analogous guarantees hold beyond the Brier loss. We answer this question affirmatively for the $\calK_1$ calibration error and all Lipschitz proper losses, including the Brier loss as a special case.

We begin by providing a calibeating result (Proposition~\ref{prop:online_offline_gap}) that extends special cases in \cite{FosterH23} to all smooth proper losses. Particularly, \cite{FosterH23} proved variants of Proposition~\ref{prop:online_offline_gap} for the squared loss and (truncated) logistic loss \eqref{eq:special_loss}, which we extend to hold for a more general family.

To state our result, we require a standard reformulation of proper losses as a representative univariate function, often called the \emph{Savage representation} \cite{Savage71, GneitingR07}.

\begin{lemma}[Theorem 1, \cite{GneitingR07}]\label{lem:savage}
Let $\ell: [0, 1] \times \{0, 1\} \to \R$ be proper, and define
\begin{equation}\label{eq:entropy_def}H(p) \defeq (1 - p)\ell(p, 0) + p\ell(p, 1) = \E_{y \sim \Bern(p)}[\ell(p, y)] \text{ for all } p \in [0, 1].\end{equation}
Then $H$ is concave. Moreover, if $H$ is differentiable, we have the equivalence
\[\ell(p, y) = H(p) + H'(p)(y - p),\text{ for all } y \in \{0, 1\}.\]
\end{lemma}

For the rest of the section, we fix a proper loss $\ell$ and identify it with its univariate Savage representation $H$. We make the following assumption on $H$ to parameterize our results.

\begin{assumption}[$G$-smoothness]\label{ass:gamma}
$H$ is differentiable, and there exists $G \ge 1$ such that for $\omega \defeq -H$,
\begin{equation}\label{eq:bregman_smooth}
  D_\omega(p \| q) \leq \frac{G}{2}(p-q)^2
 \text{ for all } (p,q) \in [0,1]^2.
\end{equation}
\end{assumption}

\begin{remark}
Assumption~\ref{ass:gamma} is only used in Proposition~\ref{prop:online_offline_gap}, and not in the recalibration step afterwards. We believe that Proposition~\ref{prop:online_offline_gap} likely extends to a broader range of losses than those satisfying Assumption~\ref{ass:gamma}: for example, \cite{FosterH23} give a somewhat more ad hoc argument that handles the cross entropy loss $\llog$, for which Assumption~\ref{ass:gamma} does not hold with $G < \infty$. We give Proposition~\ref{prop:online_offline_gap} under a smoothness bound for simplicity, as this covers the squared loss $\lsq$ as a special case.
\end{remark}

It is helpful to note that Assumption~\ref{ass:gamma} is implied by Assumption~\ref{assume:lip} with a slight blowup in parameters. However, we choose to isolate Assumption~\ref{ass:gamma} because it lets us state Proposition~\ref{prop:online_offline_gap} in a convenient way, which may be of independent future interest.

\begin{lemma}\label{lem:lip_implies_smooth}
Assumption~\ref{assume:lip} implies Assumption~\ref{ass:gamma} with $G \gets 2L$.
\end{lemma}
\begin{proof}
Let $\omega\defeq -H$ and $\psi(p)\defeq \ell(p,0)-\ell(p,1)$. For any $p,r\in[0,1]$, properness gives
\[
H(r)\le (1-r)\ell(p,0)+r\ell(p,1)
=H(p)+(r-p)(\ell(p,1)-\ell(p,0)).
\]
Equivalently,
\[
\omega(r)\ge \omega(p)+\psi(p)(r-p),
\]
so $\psi(p)\in\partial\omega(p)$ for every $p\in[0,1]$. Assumption~\ref{assume:lip} and the triangle inequality imply
\[
|\psi(p)-\psi(q)|
\le |\ell(p,0)-\ell(q,0)|+|\ell(p,1)-\ell(q,1)|
\le 2L|p-q|.
\]
Thus $\psi$ is a continuous subgradient selection of the one-dimensional convex function $\omega$, and hence $\omega$ is differentiable with $\omega'=\psi$ (using one-sided derivatives at the endpoints). By the standard equivalence between Lipschitz gradients and quadratic upper bounds for Bregman divergence (Theorem 2.1.5, \cite{Nesterov03}), $\omega$ is $2L$-smooth on $[0,1]$, which is Assumption~\ref{ass:gamma} with $G\gets 2L$.
\end{proof}

Next, to state our main helper result, analogously to the (offline) refinement score \eqref{eq:refinement}, we additionally define the \emph{online refinement score} where we use the available prefix average $\by_{t - 1}(q_t)$ at every iteration, rather than the hindsight average $\by_T(q_t)$:
\[\hcalR_\ell\Par{q_{[T]}, y_{[T]}} \defeq \frac 1 T \sum_{t \in [T]} \ell\Par{\by_{t -1 }(q_t), y_t}.\]
We now bound the gap between the online and offline refinement scores.
\begin{proposition}\label{prop:online_offline_gap}
In the setting of Model~\ref{model:hsp}, if $q_{[T]}$ is $s$-sparse and Assumption~\ref{ass:gamma} holds,
\[
  0 \leq \hcalR_\ell\Par{q_{[T]},y_{[T]}} - \calR_\ell\Par{q_{[T]},y_{[T]}}
  \leq \frac{G s\left(\log T + 1\right)}{2T}.
\]
\end{proposition}
\begin{proof}
Throughout denote $\omega \defeq -H$. We first handle the special case where $q_t=v$ for all $t\in[T]$ (i.e., $s = 1$). In this case, denoting $\by_t \defeq \by_t(v)$ for all $t \in [T]$ for simplicity,
\[
\calR_\ell(q_{[T]},y_{[T]})
=
\frac1T\sum_{t\in [T]} \ell(\bar y_T,y_t) = (1-\by_T) \ell(\by_T, 0) + \by_T\ell(\by_T, 1)
=
H(\bar y_T).
\]
Hence,
\begin{equation}\label{eq:gap_def}
\hcalR_\ell(q_{[T]},y_{[T]})-\calR_\ell(q_{[T]},y_{[T]})
=
\frac1T\sum_{t\in[T]} \ell(\bar y_{t-1},y_t)-H(\bar y_T).
\end{equation}
Moreover, for every $t\in[T]$, $\bar y_t=\bar y_{t-1}+\frac1t(y_t-\bar y_{t-1})$, 
so Lemma~\ref{lem:savage} gives
\begin{align*}
\ell(\bar y_{t-1},y_t)
&=
H(\bar y_{t-1})
+
tH'(\bar y_{t-1})(\bar y_t-\bar y_{t-1}) = tH(\by_t) - (t - 1)H(\by_{t- 1}) + tD_\omega\Par{\by_t \| \by_{t - 1}}.
\end{align*}
Telescoping the above identity over all $t \in [T]$, and plugging this into \eqref{eq:gap_def}, gives
\[
\hcalR_\ell(q_{[T]},y_{[T]})-\calR_\ell(q_{[T]},y_{[T]})
=
\frac1T\sum_{t\in[T]} tD_\omega(\bar y_t\|\bar y_{t-1}).
\]
Next, applying Assumption~\ref{ass:gamma} gives
\[
tD_\omega(\bar y_t\|\bar y_{t-1})
\le
\frac{G t}{2}(\bar y_t-\bar y_{t-1})^2
=
\frac{G}{2t}(y_t-\bar y_{t-1})^2
\le
\frac{G}{2t},
\]
because $y_t,\bar y_{t-1}\in[0,1]$. Therefore,
\[
0
\le
\hcalR_\ell(q_{[T]},y_{[T]})-\calR_\ell(q_{[T]},y_{[T]})
\le
\frac{G}{2T}\sum_{t\in[T]}\frac1t
\le
\frac{G}{2}\cdot\frac{\log T+1}{T}.
\]

For the general case, we apply the preceding argument separately to each ``bin'' $v \in [0, 1]$ such that $q_t = v$ at least once. Let $T_v:=|\{t\in[T] \mid q_t=v\}|$. Both $\calR_\ell$ and $\hcalR_\ell$ decompose over bins, so
\[
0
\le
\hcalR_\ell(q_{[T]},y_{[T]})-\calR_\ell(q_{[T]},y_{[T]})
\le
\sum_v \frac{T_v}{T}\cdot
\frac{G}{2}\cdot \frac{\log T_v+1}{T_v}.
\]
Since there are at most $s$ nonempty bins and each $T_v\le T$, the claim follows.
\end{proof}

We are now ready to prove our main result on simultaneous calibeating and calibration.

\begin{theorem}[Simultaneous calibeating and calibration]\label{thm:proper_loss_main}
Let $\ell: [0, 1] \times \{0, 1\} \to \R$ be proper such that Assumption~\ref{assume:lip} holds, and let $\eps\in(0,\half)$. In the setting of Model~\ref{model:hsp}, if $q_{[T]}$ is $s$-sparse, there is an algorithm that outputs $p_{[T]} \in [0, 1]^T$ that is $\eps$-calibrated against $y_{[T]}$ and $\eps^2$-calibeats $(q_{[T]}, y_{[T]})$ under $\ell$, if for an appropriate constant, 
\[T = \Omega\Par{\frac{L^2}{\eps^3}+ \frac{L s\log\Par{\frac{L s}{\eps}}}{\eps^2}}.\]
\end{theorem}
\begin{proof}
Define a sequence $r_t \defeq \by_{t - 1}(q_t)$, which can be computed online under Model~\ref{model:hsp}. We first observe that if $T$ is taken as stated, $r_t$ $\frac{\eps^2} 2$-calibeats $(q_{[T]}, y_{[T]})$ under $\ell$, because Proposition~\ref{prop:online_offline_gap} shows
\[\frac 1 T \sum_{t \in [T]} \ell(r_t, y_t) \le \calR_\ell\Par{q_{[T]}, y_{[T]}} + \frac{2L s(\log T + 1)}{T} \le \calR_\ell\Par{q_{[T]}, y_{[T]}} + \frac{\eps^2}{2},\]
where we took $G = 2L$ (Lemma~\ref{lem:lip_implies_smooth}).
Now to achieve the result, it is enough to ensure that $p_{[T]}$ is $(\eps, \frac{\eps^2}{2})$-recalibrated against $(r_{[T]}, y_{[T]})$, for which Theorem~\ref{thm:recalibration} suffices when $T$ is taken as stated.
\end{proof}

\begin{remark}[Sparsity in calibeating]\label{rem:sparsity}
As in prior works on calibeating \cite{FosterH23, LeeNPR22, ChenHJL26}, we assume that the hint sequence $q_{[T]}$ is $s$-sparse for some $s$ independent of $T$ (we focus on the regime $s \approx \frac 1 \eps$). We remark that some such sparsity assumption is essentially necessary. More precisely, if each $q_t$ is distinct, then each $\by_T(q_t) = y_t$ and thus the refinement score $\calR_\ell(q_{[T]}, y_{[T]}) = 0$ for any $\ell$ with $\ell(y, y) = 0$ (e.g., those in \eqref{eq:special_loss}). On the other hand, in many simple settings (e.g., $\ell = \lsq$ and each $y_t \sim \Bern(\half)$), it is impossible to obtain $\ell(p_{[T]}, y_{[T]}) = o(1)$, and thus $\alpha$-calibeating cannot occur as we take $\alpha \to 0$. On the other hand, $\alpha \to 0$ is the interesting regime to study asymptotic (re)calibration, because if $\alpha = \Omega(1)$, then writing $T$ asymptotically in $\alpha$ is ill-posed.
\end{remark}

\section{Extensions to Multiple Sequences}\label{sec:multi}

In this section, we generalize our results to the more challenging setting with multiple hint sequences, where multi-recalibration and multi-calibeating are required against \emph{every} hint sequence. We formally define the problems below. Following ideas from \cite{ChenHJL26}, we solve them using a reduction to the single-hint sequence problems we addressed in previous sections. We describe the reduction in \Cref{sec:reduction} and present our multi-sequence results in \Cref{sec:multi-recalibration,sec:multi-calibeating}.

\begin{model}[Multi-hinted sequential prediction]\label{model:mhsp}
In the \emph{multi-hinted sequential prediction} setting, for some $m \in \N$, a learner observes $m$ different \emph{hint} sequences $q_{[T]}^i \in [0, 1]^T$ for all $i \in [m]$, and outputs a \emph{prediction} sequence $p_{[T]} \in [0, 1]^T$ before observing a \emph{label} sequence $y_{[T]} \in \{0, 1\}^T$. We assume $(q_t^{[m]} \defeq \{q_t^i\}_{i \in [m]}, y_t)$, can depend jointly on all previous $(q_{[t - 1]}^{[m]}, p_{[t - 1]}, y_{[t - 1]})$, but not $p_t$; similarly, $p_t$ can depend on $(q_{[t - 1]}^{[m]}, p_{[t - 1]}, y_{[t - 1]})$ and $q_t^{[m]}$, but not $y_t$.
\end{model}

\begin{definition}[Multi-recalibration]\label{def:mrecal}
In the setting of Model~\ref{model:mhsp}, let $\alpha, \beta \ge 0$, and let $\ell: [0, 1] \times \{0, 1\} \to \R$ be a loss. We say $p_{[T]}$ is $(\alpha, \beta)$-recalibrated against $(q_{[T]}^{[m]}, y_{[T]})$ under $\ell$ if for all $i \in [m]$, $p_{[T]}$ is $(\alpha, \beta)$-recalibrated (Definition~\ref{def:recal}) against $q_{[T]}^i$ under $\ell$.
\end{definition}

\begin{definition}[Multi-calibeating]\label{def:mcalibeating}
In the setting of Model~\ref{model:mhsp}, let $\alpha \ge 0$, and let $\ell: [0, 1] \times \{0, 1\} \to \R$ be a loss. We say that $p_{[T]}$ \emph{$\alpha$-multi-calibeats} $(q_{[T]}^{[m]}, y_{[T]})$ under $\ell$ if for all $i\in [m]$, $p_{[T]}$ $\alpha$-calibeats (\Cref{def:calibeating}) $(q_{[T]}^i,y_{[T]})$ under $\ell$.
\end{definition}

\subsection{Hint compression}
\label{sec:reduction}

Following \cite{ChenHJL26}, we handle multiple hint sequences by a reduction to the single-hint sequence problems studied in previous sections. The reduction is achieved via standard results for online agnostic learning: indeed, the key component of the reduction, which we term \emph{hint compression} and formally describe in \Cref{problem:hint} below, can be viewed as equivalent to online agnostic learning.
\begin{definition}[Hint compression]
\label{problem:hint}
    In the setting of \Cref{model:mhsp}, let  $\eps_{\mathrm{hc}} \ge 0$ and let $\ell:[0,1]\times \{0,1\}\to \R$ be a loss. We say $p_{[T]}$ achieves $\eps_{\mathrm{hc}}$-hint compression against $(q_{[T]}^{[m]}, y_{[T]})$ under $\ell$ if for all $i\in [m]$,
    \[
    \E\left[\ell\left(p_{[T]}, y_{[T]}\right) - \ell \left(q_{[T]}^i, y_{[T]}\right)\right] \le \eps_{\mathrm{hc}}.
    \]
\end{definition}
\begin{proposition}[Proposition 3.2, \cite{cesa2006prediction}]
\label{prop:compression-proper}
In the setting of \Cref{model:mhsp}, let $\ell:[0,1]\times \{0,1\}\to \R$ be a proper loss satisfying \Cref{assume:lip}. There is a deterministic algorithm that outputs $p_{[T]}$ achieving $O(\frac{L\log m}{T})$-hint compression.
\end{proposition}

More precisely, Proposition 3.2 of \cite{cesa2006prediction} shows that every $\eta$-mixable loss achieves $\frac{\log m}{\eta T}$-hint compression. To justify the mixability parameter under Assumption~\ref{assume:lip}, recall from Theorem 1 of \cite{ReidW10} that a proper binary loss has a nonnegative weight function $w$ satisfying, for almost every $p\in(0,1)$,
\[\ell'(p,0)=p w(p),\qquad -\ell'(p,1)=(1-p)w(p).\]
The $L$-Lipschitz assumption therefore gives $p(1-p)w(p)\le L\min\{p,1-p\}\le L$. The curvature characterization in Section 3.1 of \cite{ERW12} then implies that the loss is $\eta$-mixable whenever $\eta p(1-p)w(p)\le 1$, so in particular it is $\eta=1/L$-mixable. We also note that Proposition~\ref{prop:compression-proper} is minimax optimal up to constant factors assuming $T = \Omega(\log m)$, by Theorem 3.6 of \cite{cesa2006prediction}.

\subsection{Multi-recalibration}
\label{sec:multi-recalibration}

In this section, we state our generalization of Theorem~\ref{thm:recalibration} to multiple hint sequences.

\begin{theorem}\label{thm:mrecalibration}
Let $\ell: [0, 1] \times \{0, 1\} \to \R$ be proper such that Assumption~\ref{assume:lip} holds, and let $\eps\in(0,\half)$. In the setting of Model~\ref{model:mhsp}, there is an algorithm that outputs $p_{[T]} \in [0, 1]^T$ that is $(\eps, \eps^2)$-multi-recalibrated against $(q_{[T]}^{[m]}, y_{[T]})$ under $\ell$, if for an appropriate constant,
\[T = \Omega\Par{\frac{L^2}{\eps^3} + \frac{L\log m}{\eps^2}}.\]
\end{theorem}
\begin{proof}
    The theorem follows from combining \Cref{prop:compression-proper} and \Cref{thm:recalibration}. 
    Concretely, we apply \Cref{prop:compression-proper} to compress the $m$ hint sequences into a single sequence, incurring $\frac{\eps^2} 2$ additional loss when $T = \Omega(L \log(m) \cdot \eps^{-2})$. We then apply \Cref{thm:recalibration} to achieve recalibration w.r.t.\ the single sequence, incurring additional loss $\frac{\eps^2} 2$ when $T = \Omega(\frac{L^2}{\eps^3})$.
\end{proof}

We make some remarks regarding the optimality of Theorem~\ref{thm:mrecalibration}. For the family of $O(1)$-Lipschitz proper losses, both terms in the iteration complexity of Theorem~\ref{thm:mrecalibration} are necessary up to constant factors. Indeed, $(\eps,\eps^2)$-multi-recalibration implies $(\eps,\eps^2)$-recalibration for a single hint sequence (by taking $m = 1$), as well as $\eps^2$-hint compression (by using the output solely for its excess loss guarantee). Thus, the minimax optimality of Theorem~\ref{thm:mrecalibration} for this family follows from our recalibration lower bound in Theorem~\ref{thm:lb} and Theorem 3.6 of \cite{cesa2006prediction}.

Interestingly, a Lipschitz assumption as in Assumption~\ref{assume:lip}, or another mixability criterion, is necessary for $(\eps, \eps^2)$-multi-recalibration in $\approx \eps^{-3}$ iterations. This is because even $\eps^2$-hint compression against arbitrary bounded proper losses requires $T = \Omega(\log m \cdot \eps^{-4})$ iterations (Theorem 3.7, \cite{cesa2006prediction}).

\subsection{Multi-calibeating and calibration}
\label{sec:multi-calibeating}

We conclude by stating an analog of Theorem~\ref{thm:proper_loss_main} for multiple hint sequences.

\begin{theorem}
\label{thm:multi-proper-loss-main}
    Let $\ell: [0, 1] \times \{0, 1\} \to \R$ be proper such that Assumption~\ref{assume:lip} holds, and let $\eps\in(0,\half)$. In the setting of Model~\ref{model:mhsp}, if $q_{[T]}^{i}$ is $s$-sparse for every $i$, there is an algorithm that outputs $p_{[T]} \in [0, 1]^T$ that is $\eps$-calibrated against $y_{[T]}$ and $\eps^2$-multi-calibeats $(q_{[T]}^{[m]}, y_{[T]})$ under $\ell$, if for an appropriate constant,
\[T = \Omega\Par{\frac{L^2}{\eps^3} + \frac{L s\log\Par{\frac{L s}{\eps}}}{\eps^2} + \frac{L\log m}{\eps^2}}.\]
\end{theorem}
\begin{proof}
    This theorem follows from combining \Cref{thm:proper_loss_main}, \Cref{prop:compression-proper} and \Cref{thm:recalibration}. Concretely, we apply \Cref{thm:proper_loss_main} to each hint sequence separately to get $m$ refined sequences that each calibeats the corresponding original sequence. We then apply \Cref{prop:compression-proper} to combine the $m$ refined sequences into a single sequence that multi-calibeats the $m$ original hint sequences. Finally, we apply \Cref{thm:recalibration} to obtain calibration while maintaining the multi-calibeating property.
\end{proof}
\section{$\calK_2$ Recalibration}\label{sec:l2recal}

In this section, we prove a root-mean-square analogue of \Cref{thm:recalibration}, and discuss its implications. The proof does not use the simultaneous Blackwell framework from \Cref{sec:blackwell}; instead, it reduces both squared $\calK_2$-calibration error and proper loss regret to the same family of quadratic tests.

We use the same notation as in \eqref{eq:net_notation}, i.e., $\calN$ is a $\frac 1 N$-separated net of $[0, 1]$ for an integer $N \in \N$ to be specified. 
For a net-valued sequence $p_{[T]}\in \calN^T$ clear from context, we define for all $z \in \calN$,
\[
n_z \defeq \Abs{\Brace{t\in[T]\mid p_t=z}},\qquad
S_z \defeq \sum_{t:p_t=z}(z-y_t),
\]
with the convention that $S_z=0$ when $n_z=0$. Then by comparing to \eqref{eq:l2cal},
\[
\calK_2^2\Par{p_{[T]},y_{[T]}}
=\frac 1 T \sum_{z\in\calN}\frac{S_z^2}{n_z},
\]
where zero-count terms are interpreted as zero.

\subsection{Quadratic identities}

In this section we derive two helpful quadratic bounds that are used to bound both the $\calK_2$-calibration error, and the proper loss regret. We begin with the former, an exact characterization of $\calK_2$.

\begin{lemma}[Quadratic representation of $\calK_2^2$]\label{lem:l2_fenchel}
For any $p_{[T]} \in\calN^T$, $y_{[T]} \in \{0, 1\}^T$,
\[
T\calK_2^2\Par{p_{[T]},y_{[T]}}
=
\sup_{\boldsymbol{\lambda}\in[-1,1]^{\calN}}
\sum_{t\in[T]}\Par{2\vlam_{p_t}(p_t-y_t)-\vlam_{p_t}^2}.
\]
\end{lemma}
\begin{proof}
For a fixed bin $z \in \calN$, the contribution to the right-hand side is
\[
\sup_{\lambda_z\in[-1,1]}\Par{2\lambda_z S_z-n_z\lambda_z^2}.
\]
If $n_z=0$, this value is zero. If $n_z>0$, the unconstrained maximizer is $\lambda_z=S_z/n_z$. Since every summand $z-y_t$ lies in $[-1,1]$, the constraint $\lambda_z\in[-1,1]$ is inactive, and the value is $S_z^2/n_z$. Summing over $z\in\calN$ proves the claim.
\end{proof}

We next derive a similar bound for the proper loss regret with respect to $\ell$. To state this bound, we follow notation from Section~\ref{sec:calibeat} (particularly, Lemma~\ref{lem:savage}, Assumption~\ref{ass:gamma}, and Lemma~\ref{lem:lip_implies_smooth}), and define
\begin{equation}\label{eq:savage_notation}
H(p) \defeq (1-p)\ell(p,0)+p\ell(p,1),\quad
\omega(p)\defeq -H(p),\quad
\psi(p)\defeq \omega'(p) = \ell(p,0)-\ell(p,1).
\end{equation}
For the rest of the section, we parameterize our proper loss regret using $G$, the Lipschitz constant of $\psi \defeq \omega'$ (Assumption~\ref{ass:gamma}). Recall that $G \le 2L$ under Assumption~\ref{assume:lip}, by Lemma~\ref{lem:lip_implies_smooth}.

\begin{corollary}[Proper loss regret is dominated by a quadratic audit]\label{cor:l2_regret_audit}
Under Assumption~\ref{ass:gamma}, define
\[
f_z(q)\defeq\frac{\psi(z)-\psi(q)}{G}\text{ for all } z,q \in [0, 1].
\]
Then $f_z(q)\in[-1,1]$, and for all $y\in\{0,1\}$,
\[
\ell(z,y)-\ell(q,y)
\le
\frac G 2\Par{2f_z(q)(z-y)-f_z(q)^2}.
\]
\end{corollary}
\begin{proof}
The bound $|f_z(q)|\le 1$ follows from $G$-Lipschitzness and $|z-q|\le 1$. Next,
\[\ell(z, y) - \ell(q, y) = \Par{\psi(z) - \psi(q)}\Par{z - y}- D_\omega\Par{z \| q}\]
follows by expanding definitions. The claimed inequality follows by using co-coercivity of the gradient when $\psi$ is Lipschitz (see Eq.\ (2.1.7), \cite{Nesterov03}), which gives
\[
\ell(z,y)-\ell(q,y)
\le
\Par{\psi(z)-\psi(q)}(z-y)
-\frac{\Par{\psi(z)-\psi(q)}^2}{2G}.
\]
\end{proof}

\subsection{Quadratic audit primitive}

Our earlier Lemma~\ref{lem:l2_fenchel} and Corollary~\ref{cor:l2_regret_audit} show how to control $\calK_2$-calibration error and proper loss regret via upper bounds of the form
\[
\Phi(u,e)\defeq 2ue-u^2,
\]
where eventually, $e \gets z - y$ and $u \gets \lam f_z(q)$ for a prediction $z \in \calN$ and hint $q \in [0, 1]$. In particular, Lemma~\ref{lem:l2_fenchel} uses  $f_z(q) \equiv 1$ in its upper bound, and similarly Corollary~\ref{cor:l2_regret_audit} uses $\lam \equiv 1$.

Our main result in this section (Lemma~\ref{lem:l2_contextual_audit}) therefore shows how to control an aggregated $\Phi(\lam f_z(q), z -y)$ for $f_z$ belonging to a specified family of functions, and $\lam \in [-1, 1]$, to simultaneously achieve $\calK_2$-calibration error and proper loss regret bounds.
Our main helper tool in obtaining this result is the exponentially weighted forecaster for exp-concave losses \cite{Vovk90,cesa2006prediction}.
\begin{lemma}[Aggregation for exp-concave losses]\label{lem:l2_expconcave}
Fix a horizon $S\in\N$ and let $\calE$ be a finite set of experts. On each round $s\in[S]$, expert $e\in\calE$ predicts $x_{e,s}$ in a convex set $\calX$, the learner predicts a convex combination $\hat x_s=\sum_{e\in\calE}w_{e,s}x_{e,s}$, and then a loss function $L_s:\calX\to\R$ is revealed. Suppose every $L_s$ is $\eta$-exp-concave, meaning $x\mapsto \exp(-\eta L_s(x))$ is concave on $\calX$. If the weights are initialized uniformly and updated by
\[
w_{e,s+1}
\gets
\frac{w_{e,s}\exp(-\eta L_s(x_{e,s}))}{\sum_{e'\in\calE}w_{e',s}\exp(-\eta L_s(x_{e',s}))},
\]
then, for every $e^\star\in\calE$,
\[
\sum_{s\in[S]} L_s(\hat x_s)
\le
\sum_{s\in[S]} L_s(x_{e^\star,s})
+\frac 1\eta\log|\calE|.
\]
\end{lemma}
\begin{proof}
We briefly recall the standard proof strategy, deferring a more detailed exposition to Theorem 3.2, \cite{cesa2006prediction}.
By exp-concavity and Jensen's inequality,
\[
\exp(-\eta L_s(\hat x_s))
\ge
\sum_{e\in\calE}w_{e,s}\exp(-\eta L_s(x_{e,s})).
\]
Multiplying this inequality over $s$ and comparing the final unnormalized total weight to the weight of a fixed expert $e^\star$ gives the stated bound.
\end{proof}

For our application, $\calX=[-1,1]$ and $L_s(u)=(u-e_s)^2$ with $e_s\in[-1,1]$. This loss is $\frac 1 8 $-exp-concave on $[-1,1]$, because the second derivative of $\exp(-\eta(u-e_s)^2)$ is nonpositive whenever $\eta\le 1/8$ and $|u-e_s|\le 2$. Thus \Cref{lem:l2_expconcave} gives square-loss regret at most $8\log|\calE|$. As we next show, this lets us achieve our desired bounds on aggregations of $\Phi$, using the identity
\begin{equation}\label{eq:l2_phi_square}
\Phi(u,e)=e^2-(u-e)^2.
\end{equation}

\begin{lemma}[Contextual quadratic audit]\label{lem:l2_contextual_audit}
Fix $T\ge 1$ and a finite class $\calF\subseteq\{f:[0,1]\to[-1,1]\}$, and follow the notation \eqref{eq:net_notation}. There is an online procedure that, after seeing $q_t$, outputs $u_{z,t}\in[-1,1]$ for every $z\in\calN$, and has the pathwise guarantee for every $(q_{[T]}, p_{[T]}, y_{[T]}) \in [0, 1]^T \times \calN^T \times \{0, 1\}^T$,
\[
\sum_{z\in\calN}
\sup_{\substack{f\in\calF \\ \lambda\in[-1,1]}}
\sum_{t:p_t=z}\Phi(\lambda f(q_t),z-y_t)
\le
\sum_{t\in[T]}\Phi(u_{p_t,t},p_t-y_t)
+16N\log(9T|\calF|).
\]
\end{lemma}
\begin{proof}
We run one independent copy of Lemma~\ref{lem:l2_expconcave} for each $z \in \calN$, using the same expert set $\calE$:
\[
\calE\defeq \calF\times\Lambda_T,\qquad
\Lambda_T\defeq\Brace{-1+\frac k T\mid k=0,1,\ldots,2T}.
\]
Each independent copy maintains its own independent iteration clock $s$, and we only increment this clock on iterations where the realized prediction is $p_t = z$. For the $z^{\text{th}}$ copy, we define $L_s(x) \defeq (x - (z - y_t))^2$, where $y_t$ is the label on the $s^{\text{th}}$ iteration where the prediction $p_t = z$ is made. Each expert $e = (f, \lam) \in \calE$ makes the prediction $x_{e, s} \defeq \lam f(q_t) \in [-1, 1]$, without knowledge of $p_t$, $y_t$. We let $u_{z, t}$ be the aggregated prediction made by Lemma~\ref{lem:l2_expconcave} based on these predictions.

Applying \Cref{lem:l2_expconcave} to the active rounds of bin $z$, with $\eta = \frac 1 8$, gives for every $(f,\lambda)\in\calF\times\Lambda_T$,
\begin{align*}
\sum_{t:p_t=z}(u_{z,t}-(z-y_t))^2
&\le
\sum_{t:p_t=z}(\lambda f(q_t)-(z-y_t))^2
+8\log\Par{|\calF|(2T+1)} \\
&\le \sum_{t:p_t=z}(\lambda f(q_t)-(z-y_t))^2
+8\log\Par{3T|\calF|}.
\end{align*}
Using \eqref{eq:l2_phi_square}, this is equivalent to
\[
\sum_{t:p_t=z}\Phi(\lambda f(q_t),z-y_t)
\le
\sum_{t:p_t=z}\Phi(u_{z,t},z-y_t)
+8\log\Par{3T|\calF|}.
\]
It remains to pass from $\Lambda_T$ to all $\lambda\in[-1,1]$. The derivative of $\Phi(\lambda f(q),e)$ with respect to $\lambda$ is
\[
2f(q)e-2\lambda f(q)^2,
\]
whose absolute value is at most $4$. Every $\lambda^\star\in[-1,1]$ is within distance at most $\frac 1 {2T}$ of a grid point in $\Lambda_T$, so each term changes by at most $\frac 2 T$, and the total cost over at most $T$ active rounds is at most $2$, and $8\log(3T|\calF|) + 2 \le 8\log(9T|\calF|)$. Summing over all $N+1 \le 2N$ bins proves the claim.
\end{proof}

Finally, we use the following oracle to define our predictions $p_t$, made after observing the bin aggregated predictions $u_{z,t}$ in Lemma~\ref{lem:l2_contextual_audit}, which in turn only depend on the hint $q_t$.

\begin{lemma}\label{lem:l2_oracle}
Fix $\{u_z \in[-1,1]\}_{z \in \calN}$. We can compute $\va\in\Delta^{\calN}$ in $O(|\calN|)$ time such that
\[
\sum_{z\in\calN}\va_z\Par{2u_z(z-y)-u_z^2}\le \frac{1}{4N^2} \text{ for all } y \in \{0, 1\}.
\]
\end{lemma}
\begin{proof}
This is a finite zero-sum game. By von Neumann's minimax theorem, it suffices to show that for every mixed label distribution, there is a pure grid prediction with small expected payoff. Let the distribution satisfy $\Pr[y=1]=\mu$, and choose the closest $z\in\calN$ to $\mu$. Then $|z-\mu|\le \frac 1 {2N}$, and
\[
\E_y\Brack{2u_z(z-y)-u_z^2}
=2u_z(z-\mu)-u_z^2
\le (z-\mu)^2
\le \frac{1}{4N^2},
\]
where we used $2ab-b^2\le a^2$ with $a=z-\mu$ and $b=u_z$. The runtime comes from solving a linear program with dimensions $2$ and $O(|\calN|)$, and applying the result of \cite{Megiddo84}.
\end{proof}

\subsection{Main theorem}

We are now ready to complete the proof of our main result on $\calK_2$ recalibration.

\begin{theorem}\label{thm:l2recalibration}
Let $\ell: [0, 1] \times \{0, 1\} \to \R$ be proper such that Assumption~\ref{assume:lip} holds, and let $\eps \in (0, \half)$. In the setting of Model~\ref{model:hsp}, there is an algorithm that outputs $p_{[T]} \in [0, 1]^T$ that is $(\eps, \eps^2)$-$\calK_2$-recalibrated against $(q_{[T]}, y_{[T]})$ under $\ell$, if for an appropriate constant,
\[T = \Omega\Par{\frac{L^{\frac 3 2}}{\eps^3}\log\Par{\frac L \eps}}.\]
\end{theorem}
\begin{proof}
Following \eqref{eq:savage_notation} and Lemma~\ref{lem:lip_implies_smooth}, we let $G \defeq 2L$. We define the finite audit class
\[
\calF_\ell\defeq \Brace{\iota}\cup\Brace{f_z:z\in\calN},
\quad
\iota(q)\defeq 1,\quad
f_z(q)\defeq\frac{\psi(z)-\psi(q)}{G}.
\]
Every member of $\calF_\ell$ maps $[0,1]$ to $[-1,1]$, and $|\calF_\ell|\le N+2 \le 3N$.
Instantiate the procedure in \Cref{lem:l2_contextual_audit} for $\calN$ and $\calF_\ell$. On round $t$, after observing $q_t$, obtain $u_{z,t}\in[-1,1]$ for all $z\in\calN$. Apply \Cref{lem:l2_oracle} to the vector $(u_{z,t})_{z\in\calN}$, sample $p_t$ from the resulting distribution, observe $y_t$, and update the audit aggregator corresponding to the realized bin $p_t$ with outcome $p_t-y_t$.

Because Model~\ref{model:hsp} does not observe the fresh draw $p_t$ before fixing $(q_t,y_t)$, the oracle gives
\[
\E\Brack{\Phi(u_{p_t,t},p_t-y_t)\mid\text{past},q_t,y_t}
\le \frac{1}{4N^2}.
\]
Summing over time,
\begin{equation}\label{eq:l2_own_payoff}
\E\Brack{\sum_{t\in[T]}\Phi(u_{p_t,t},p_t-y_t)}
\le \frac{T}{4N^2}.
\end{equation}
Combining \Cref{lem:l2_contextual_audit} with \eqref{eq:l2_own_payoff} and $|\calF_\ell|=N+2$ gives the master audit bound
\begin{equation}\label{eq:l2_master_bound}
\E\Brack{
\sum_{z\in\calN}
\sup_{f\in\calF_\ell,\lambda\in[-1,1]}
\sum_{t:p_t=z}\Phi(\lambda f(q_t),z-y_t)}
\le
\frac{T}{4N^2}+16N\log\Par{27NT}.
\end{equation}
By choosing the constant $\iota$ function in each bin $z$, \eqref{eq:l2_master_bound} and \Cref{lem:l2_fenchel} imply
\[
T\E\Brack{\calK_2^2\Par{p_{[T]},y_{[T]}}}
\le
\frac{T}{4N^2}+16N\log(27NT),
\]
which gives an expected $\calK_2^2$ of $\le \eps^2$ by taking $N \ge \frac 1 \eps$ and $T = \Omega(\frac 1 {\eps^3}\log(\frac 1 \eps))$. Jensen's inequality then shows this also implies the expected $\calK_2$ is $\le \eps$ as needed for $\calK_2$-recalibration.

For the proper loss regret, in each realized bin $z$ we choose $f_z$ and $\lambda=1$. Then 
\begin{align*}
\E\Brack{\frac 1 T \sum_{t \in [T]} \ell(p_t, y_t) - \ell(q_t, y_t)} &\le \E\Brack{\frac{G}{2T} \sum_{t \in [T]} \Phi\Par{f_{p_t}(q_t), p_t - y_t}} \\
&\le \frac{G}{8N^2} + \frac{8GN\log(27NT)}{T},
\end{align*}
where we used Corollary~\ref{cor:l2_regret_audit} in the first line. By taking $N = \Omega(\sqrt{L}/\eps)$ and taking $T$ as in the theorem statement, both of the terms are at most $\frac{\eps^2}{2}$, and thus the proper loss regret is $\le \eps^2$ as desired.\end{proof}

We note that by appropriately combining Theorem~\ref{thm:l2recalibration} (in place of Theorem~\ref{thm:recalibration}) with the reductions in Sections~\ref{sec:calibeat} and~\ref{sec:multi}, it is straightforward to derive $\calK_2$ variants of our calibeating and multi-sequence results. As an example, we state a $\calK_2$ calibeating analog of Theorem~\ref{thm:proper_loss_main}.

\begin{theorem}[Simultaneous calibeating and $\calK_2$-calibration]\label{thm:l2_calibeating}
Let $\ell: [0,1]\times\{0,1\}\to\R$ be proper such that Assumption~\ref{assume:lip} holds, and let $\eps\in(0,\half)$. In the setting of Model~\ref{model:hsp}, if $q_{[T]}$ is $s$-sparse, there is an algorithm that outputs $p_{[T]}\in[0,1]^T$ that is $\eps$-$\calK_2$-calibrated against $y_{[T]}$ and $\eps^2$-calibeats $(q_{[T]}, y_{[T]})$ under $\ell$,
if for an appropriate constant,
\[
T=\Omega\Par{\frac{L^{\frac 3 2}}{\eps^3}\log\Par{\frac{L}{\eps}}+\frac{L s\log\Par{\frac{L s}{\eps}}}{\eps^2}}.
\]
\end{theorem}

\subsection{Multi-hinted extensions}\label{sec:l2-multi}

For this subsection, we say that $p_{[T]}$ is $(\alpha,\beta)$-multi-$\calK_2$-recalibrated against $(q_{[T]}^{[m]},y_{[T]})$ under $\ell$ if, for every $i\in[m]$, it is $(\alpha,\beta)$-$\calK_2$-recalibrated against $(q_{[T]}^i,y_{[T]})$ under $\ell$ in the sense of Definition~\ref{def:recal}.

\begin{theorem}[Multi-recalibration with $\calK_2$ calibration]\label{thm:l2_mrecalibration}
Let $\ell: [0,1]\times\{0,1\}\to\R$ be proper such that Assumption~\ref{assume:lip} holds, and let $\eps\in(0,\half)$. In the setting of Model~\ref{model:mhsp}, there is an algorithm that outputs $p_{[T]}\in[0,1]^T$ that is $(\eps,\eps^2)$-multi-$\calK_2$-recalibrated against $(q_{[T]}^{[m]},y_{[T]})$ under $\ell$, if for an appropriate constant,
\[
T=\Omega\Par{\frac{L^{\frac 3 2}}{\eps^3}\log\Par{\frac{L}{\eps}}+\frac{L\log m}{\eps^2}}.
\]
\end{theorem}
\begin{proof}
This is the same reduction as in \Cref{thm:mrecalibration}, with \Cref{thm:l2recalibration} replacing \Cref{thm:recalibration}. First run the hint-compression algorithm of \Cref{prop:compression-proper} on the $m$ hint sequences to produce a single predictable sequence $r_{[T]}$ such that, for every $i\in[m]$,
\[
\E\Brack{\ell\Par{r_{[T]},y_{[T]}}-\ell\Par{q^i_{[T]},y_{[T]}}}\le \frac{\eps^2}{2},
\]
provided $T=\Omega(L\log(m)/\eps^2)$. Then apply \Cref{thm:l2recalibration} to $r_{[T]}$, with constants adjusted so that
\[
\E\Brack{\calK_2\Par{p_{[T]},y_{[T]}}}\le \eps,\qquad
\E\Brack{\ell\Par{p_{[T]},y_{[T]}}-\ell\Par{r_{[T]},y_{[T]}}}\le \frac{\eps^2}{2},
\]
which requires $T=\Omega(L^{3/2}\eps^{-3}\log(L/\eps))$. Combining the two inequalities gives the desired excess-loss guarantee against every $q^i_{[T]}$, and the $\calK_2$ calibration guarantee is common to all $i$.
\end{proof}

\begin{theorem}[Multi-calibeating and $\calK_2$ calibration]\label{thm:l2_multi_calibeating}
Let $\ell: [0,1]\times\{0,1\}\to\R$ be proper such that Assumption~\ref{assume:lip} holds, and let $\eps\in(0,\half)$. In the setting of Model~\ref{model:mhsp}, if $q_{[T]}^i$ is $s$-sparse for every $i\in[m]$, there is an algorithm that outputs $p_{[T]}\in[0,1]^T$ that is $\eps$-$\calK_2$-calibrated against $y_{[T]}$ and $\eps^2$-multi-calibeats $(q_{[T]}^{[m]},y_{[T]})$ under $\ell$, if for an appropriate constant,
\[
T=\Omega\Par{\frac{L^{\frac 3 2}}{\eps^3}\log\Par{\frac{L}{\eps}}+\frac{L s\log\Par{\frac{L s}{\eps}}}{\eps^2}+\frac{L\log m}{\eps^2}}.
\]
\end{theorem}
\begin{proof}
For each $i\in[m]$, let $\by^i_{t-1}(v)$ denote the empirical label average over rounds $\tau<t$ with $q^i_\tau=v$, using the same zero-count convention as in Definition~\ref{def:calerr}, and define the online refinement $r_t^i\defeq \by^i_{t-1}(q_t^i)$. By \Cref{prop:online_offline_gap} and Lemma~\ref{lem:lip_implies_smooth}, the sparsity term in the theorem statement ensures that, for every $i\in[m]$,
\[
\E\Brack{\ell\Par{r^i_{[T]},y_{[T]}}}\le
\E\Brack{\calR_\ell\Par{q^i_{[T]},y_{[T]}}}+\frac{\eps^2}{3}.
\]
Next apply \Cref{prop:compression-proper} to the $m$ refined sequences $r^1_{[T]},\ldots,r^m_{[T]}$ to obtain a single predictable sequence $r_{[T]}$ satisfying, for every $i\in[m]$,
\[
\E\Brack{\ell\Par{r_{[T]},y_{[T]}}-\ell\Par{r^i_{[T]},y_{[T]}}}\le \frac{\eps^2}{3},
\]
which accounts for the $L\log(m)/\eps^2$ term. Finally, apply \Cref{thm:l2recalibration} to $r_{[T]}$, with constants adjusted so that $p_{[T]}$ is $\eps$-$\calK_2$-calibrated and
\[
\E\Brack{\ell\Par{p_{[T]},y_{[T]}}-\ell\Par{r_{[T]},y_{[T]}}}\le \frac{\eps^2}{3}.
\]
Adding the three inequalities shows that, for every $i\in[m]$,
\[
\E\Brack{\ell\Par{p_{[T]},y_{[T]}}}\le
\E\Brack{\calR_\ell\Par{q^i_{[T]},y_{[T]}}}+\eps^2,
\]
which is exactly $\eps^2$-multi-calibeating. The $\calK_2$ calibration conclusion follows from the final recalibration step.
\end{proof}

\section{Lower Bound}\label{sec:lower}

In this section, we give a hard instance that demonstrates the recalibration rates achieved by our algorithm in Section~\ref{sec:recal} are optimal up to constant factors. Throughout, we fix a target $\eps \in (0, \half)$ and lower bound the number of iterations required for an online algorithm to achieve $(\eps, \eps^2)$-recalibration (against the squared loss). For notational convenience, in this section we define
\begin{equation}\label{eq:net}\calN \defeq \Brace{i\eps}_{i \in [\lfloor \eps^{-1}\rfloor]}\end{equation}
to be an $\eps$-net of $[0, 1]$. We also define the \emph{Brier score} (average squared loss)
\begin{equation}\label{eq:brier}\calB_2\Par{p_{[T]}, y_{[T]}} \defeq \frac 1 T \sum_{t \in [T]} (p_t - y_t)^2\end{equation}
for any $T$-length sequences $p_{[T]} \in [0, 1]^T$, $y_{[T]} \in \{0, 1\}^T$.
We now outline our hard instance.

\begin{construction}\label{construction:lb}
Let $\eps \in (0, \half)$ and let $T \in \N$. Following notation \eqref{eq:net}, let $\calM \defeq \calN \cap [\frac 1 4, \frac 3 4] = \{v_i\}_{i \in [m]}$ where $m \defeq |\calM|$. Our hard instance is as follows: for $t \in [T]$, we let
\[q_t = v_i \text{ where } i \equiv t \pmod m,\quad y_t \sim \Bern(q_t).\]
The labels are drawn independently over $t$ and independently of the learner's internal randomness.
\end{construction}

In other words, for $L \defeq \lceil\frac T m\rceil$, we loop through all elements of $\calM$ each for $L - 1$ or $L$ times in total, set $q_t = v_i$ for an appropriate index $i \in [m]$, and sample $y_t \sim \Bern(q_t)$. This is the same hard instance as is described in Section 3.1.1, \cite{CollinaLNR26}.

Observe that the initial predictions $\{q_t\}_{t \in [T]}$ truthfully capture the generation process for the labels $\{y_t\}_{t \in [T]}$. Moreover, their expected calibration error satisfies
\[\E\Brack{\calK_1(q_{[T]}, y_{[T]})} = O\Par{\frac 1 {\sqrt L}}\]
because the expected absolute fluctuation in each bin is $O(\sqrt L)$, and each $v_i$ is set to $q_t$ for either $L - 1$ or $L$ times. Thus, taking $L=\Theta(\eps^{-2})$, equivalently $T=\Theta(\eps^{-3})$, makes $q_{[T]}$ $\eps$-calibrated in expectation, so setting $p_{[T]}=q_{[T]}$ is an $(\eps,0)$-recalibration in expectation.

We show that there is (asymptotically) no better strategy.
In Section~\ref{ssec:rounding} we first show that up to negligible loss, we may assume that all predictions are in $\calN$. We then lower bound the rounded miscalibration error in Section~\ref{ssec:miscal}, with our main result stated as Theorem~\ref{thm:lb}.

\subsection{Rounding}\label{ssec:rounding}

Throughout this section, let $\calF_t$ be the filtration generated by $\{q_{[t]},p_{[t-1]},y_{[t-1]}\}$ together with the learner's internal randomness used up to and including the draw of $p_t$. We assume this internal randomness is independent of the label draws. Thus, the possibly randomized prediction $p_t$ is $\calF_t$-measurable and $\E[y_t\mid\calF_t]=q_t$. We similarly let $\calF'_t$ augment $\calF_t$ with $y_t$, i.e., it contains the history through the $t^{\text{th}}$ prediction and label.
We also define
\begin{equation}\label{eq:round_def}r_t \defeq \arg\min_{v \in \calN} |v - p_t|\end{equation}
to be the deterministic rounding of $p_t$ to our net $\calN$, breaking ties arbitrarily but consistently. 

We next relate the calibration error and squared loss regret of the rounded predictions $r_{[T]}$ to those achieved by the predictions $p_{[T]}$, allowing us to focus on the former sequence.

\begin{lemma}\label{lem:round_cal_err}
For any $p_{[T]} \in [0, 1]^T$, $y_{[T]} \in \{0, 1\}^T$, following notation \eqref{eq:round_def}, 
\[\calK_1\Par{r_{[T]}, y_{[T]}} \le \calK_1\Par{p_{[T]}, y_{[T]}} + \eps.\]
\end{lemma}
\begin{proof}
Throughout this proof, for each $v \in \calN$, let $I_v \subseteq [0, 1]$ denote the set of possible $p$ such that $r_t = v$ if $p_t = p$ (i.e., the set of $p$ that would be rounded to $v$ using \eqref{eq:round_def}). We have
\begin{align*}
\calK_1\Par{r_{[T]}, y_{[T]}} &= \frac 1 T \sum_{v \in \calN} \Abs{\sum_{t \in [T]: r_t = v} (v - y_t)} = \frac 1 T \sum_{v \in \calN} \Abs{\sum_{t \in [T]: p_t \in I_v} (v - y_t)} \\
&\le \frac 1 T \sum_{v \in \calN} \Par{\Abs{\sum_{t \in [T]: p_t \in I_v} (p_t - y_t)} + \sum_{t \in [T]: p_t \in I_v} |p_t - v|} \\
&\le \frac 1 T \sum_{v \in \calN} \Abs{\sum_{t \in [T]: p_t \in I_v} (p_t - y_t)} + \eps.
\end{align*}
Here, the last inequality used that every $p_t \in I_v$ has $|p_t - v| \le \eps$. Next, for each $v \in \calN$,
\[\Abs{\sum_{t \in [T]: p_t \in I_v} (p_t - y_t)} \le \sum_{p \in I_v} \Abs{\sum_{t \in [T]: p_t = p} (p - y_t)}\]
by the triangle inequality. Finally, combining the above two displays gives the claim.
\end{proof}

In the remainder of the section, any unconditional expectations are always with respect to $\calF'_T$.

\begin{lemma}\label{lem:round_loss_err}
For any $p_{[T]} \in [0, 1]^T$, following Construction~\ref{construction:lb} and notation \eqref{eq:round_def},
\[\E\Brack{\calB_2\Par{r_{[T]}, y_{[T]}} - \calB_2\Par{q_{[T]}, y_{[T]}}} \le 2\E\Brack{\calB_2\Par{p_{[T]}, y_{[T]}} - \calB_2\Par{q_{[T]}, y_{[T]}}} + 2\eps^2.\]
\end{lemma}
\begin{proof}
Because of the bias-variance decomposition
\begin{equation}\label{eq:bias_var_decomp}\E\Brack{\Par{p_t - y_t}^2 \mid \calF_t} = \E\Brack{\Par{q_t - y_t}^2  \mid \calF_t} + \E\Brack{\Par{p_t - q_t}^2  \mid \calF_t},\end{equation}
and a similar decomposition for $\E[(r_t - y_t)^2\mid \calF_t]$, it is enough to show that for every $t \in [T]$,
\begin{align*}
\E\Brack{(r_t - q_t)^2 \mid \calF_t} \le 2\E\Brack{(p_t - q_t)^2 \mid \calF_t} + 2\eps^2.
\end{align*}
This follows deterministically from $(a + b)^2 \le 2a^2 + 2b^2$, and $|p_t - r_t| \le \eps$.
\end{proof}

Combining Lemmas~\ref{lem:round_cal_err} and~\ref{lem:round_loss_err}, we have shown that any $(\eps, \eps^2)$-recalibration algorithm playing predictions $p_{[T]} \in [0, 1]^T$ implies a $(2\eps, 4\eps^2)$-recalibration algorithm playing predictions $r_{[T]} \in \calN^T$. Henceforth we focus on arguing the latter goal is impossible if $T = o(\eps^{-3})$.

\subsection{Lower bounding miscalibration}\label{ssec:miscal}

In this section, let $r_t \in \calN$ be a $\calF_t$-measurable prediction restricted to being on the net $\calN$, for all $t \in [T]$. Our goal is to lower bound $\calK_1(r_{[T]}, y_{[T]})$ assuming that $r_{[T]}$ achieves good Brier score. Our starting point is a bias-miscalibration decomposition (Lemma~\ref{lem:lower_k1_r}), where we define
\begin{equation}\label{eq:nvdef}\begin{aligned}\nu_v \defeq \sum_{t \in [T]} \ind_{r_t = v}, \quad
\mu_v \defeq \sum_{t \in [T]} \ind_{r_t = v} (q_t - y_t), \end{aligned}\end{equation}
for all $v \in \calN$. Note that $\nu_v$ and $\mu_v$ are $\calF'_T$-measurable random variables, that respectively correspond to the times $r_t = v$ was chosen, and the total miscalibration of $q_{[T]}, y_{[T]}$ on those rounds. 

\begin{lemma}\label{lem:lower_k1_r}
Following Construction~\ref{construction:lb} and notation \eqref{eq:round_def}, \eqref{eq:nvdef}, 
\[\E\Brack{\calK_1\Par{r_{[T]}, y_{[T]}}} \ge \frac 1 T \sum_{v \in \calN} \E\Brack{|\mu_v|} - \sqrt{\E[\calB_2(r_{[T]}, y_{[T]}) - \calB_2(q_{[T]}, y_{[T]})]}.\]
\end{lemma}
\begin{proof}
In addition to \eqref{eq:nvdef} we also denote the per-prediction bias by
\[\beta_v \defeq \sum_{t \in [T]} \ind_{r_t = v}(r_t - q_t),\]
for all $v \in \calN$. Then, 
\begin{align*}
\calK_1\Par{r_{[T]}, y_{[T]}} = \frac 1 T \sum_{v \in \calN} \Abs{\beta_v + \mu_v} \ge \frac 1 T \sum_{v \in \calN} |\mu_v| - \frac 1 T \sum_{v \in \calN} |\beta_v|.
\end{align*}
By taking expectations, and bounding
\begin{align*}\E\Brack{\frac 1 T \sum_{v \in \calN} |\beta_v|} &\le \frac 1 T \E\Brack{\sum_{v \in \calN} \sqrt{\nu_v \cdot \sum_{t \in [T]} \ind_{r_t = v} (r_t - q_t)^2}} \\
&\le \frac 1 T \E\Brack{\sqrt{\Par{\sum_{v \in \calN} \nu_v}\Par{\sum_{t \in [T]} (r_t - q_t)^2}}} \\
&= \E\Brack{\sqrt{\frac 1 T \sum_{t \in [T]} (r_t - q_t)^2}} \le \sqrt{\E\Brack{\frac 1 T \sum_{t \in [T]}(r_t - q_t)^2}},
\end{align*}
we have the claim upon applying the decomposition \eqref{eq:bias_var_decomp}. The first and second inequalities above were the Cauchy-Schwarz inequality, and the last applied Jensen's inequality.
\end{proof}

Lemma~\ref{lem:lower_k1_r} shows that to rule out a rounded predictor sequence $r_{[T]} \in \calN^T$ achieving $(\eps, \eps^2)$-recalibration (in expectation) where $T = o(\eps^{-3})$, it is enough to show 
\[\sum_{v \in \calN} \E[|\mu_v|] = \Omega(\eps^{-2}).\]
We next develop tools for lower bounding $\sum_{v \in \calN} \E[|\mu_v|]$. We begin with a technical helper result.

\begin{restatable}{lemma}{restatefourthmoment}\label{lem:fourthmoment}
Following Construction~\ref{construction:lb} and notation \eqref{eq:round_def}, \eqref{eq:nvdef}, for all $v \in \calN$,
\[\E\Brack{\mu_v^4} \le 66(\E\Brack{\nu_v^2} + 1).\]
\end{restatable}
\begin{proof}
To ease notation, for all $t \in [T]$ let
\[I_t \defeq \ind_{r_t = v},\quad \Delta_t \defeq I_t(q_t - y_t),\quad a_t\defeq \E[\Delta_t^2\mid\calF_t].\]
Then $\E[\Delta_t \mid \calF_t] = 0$, $|\Delta_t|\le 1$, and, as $q_t \in [\frac 1 4, \frac 3 4]$ always,
\begin{equation}\label{eq:23_moment_bounds}a_t = q_t(1 - q_t)I_t \le \frac 1 4 I_t,\quad \Abs{\E[\Delta_t^3 \mid \calF_t]} = q_t(1 - q_t)\Abs{2q_t - 1}I_t \le a_t.\end{equation}
Define $\sigma_t \defeq \sum_{s \in [t]}\Delta_s$, $V_t\defeq\sum_{s\in[t]}a_s$, and $\sigma_T^*\defeq\max_{0\le t\le T}|\sigma_t|$, so that $\mu_v = \sigma_T$. The process $\{\sigma_t\}_{t=0}^T$ is a martingale with respect to $\{\calF'_t\}_{t=0}^T$. Since $\E[\Delta_t^4\mid\calF_t]\le a_t$ and $4|x|\le 2x^2+2$, we have
\begin{align*}
\E\Brack{\sigma_t^4 - \sigma_{t - 1}^4 \mid \calF_t}
&= 6\sigma_{t-1}^2 a_t + 4\sigma_{t - 1}\E\Brack{\Delta_t^3 \mid \calF_t} + \E\Brack{\Delta_t^4 \mid \calF_t} \\
&\le 8\sigma_{t-1}^2a_t+3a_t.
\end{align*}
Writing $X\defeq\E[\sigma_T^4]$ and summing over $t$ therefore gives
\begin{equation}\label{eq:fourth_recursion}
X\le 8\E\Brack{\sum_{t\in[T]}\sigma_{t-1}^2a_t}+3\E[V_T]
\le 8\E\Brack{(\sigma_T^*)^2V_T}+3\E[V_T].
\end{equation}
By Cauchy--Schwarz and Doob's $L^4$ maximal inequality,
\begin{align*}
\E\Brack{(\sigma_T^*)^2V_T}
&\le \sqrt{\E[(\sigma_T^*)^4]\E[V_T^2]}
\le \frac{16}{9}\sqrt{X\E[V_T^2]}
\le \frac49\sqrt{X\E[\nu_v^2]},
\end{align*}
where the last inequality uses $V_T\le \nu_v/4$, which follows from \eqref{eq:23_moment_bounds}. Combining this bound with \eqref{eq:fourth_recursion}, using $\E[V_T]\le\E[\nu_v]/4$, and then applying Young's inequality yields
\begin{align*}
X
&\le \frac{32}{9}\sqrt{X\E[\nu_v^2]}+\frac34\E[\nu_v] \\
&\le \frac12X+\frac{512}{81}\E[\nu_v^2]+\frac38\Par{\E[\nu_v^2]+1}.
\end{align*}
Rearranging proves the stated bound (in fact, with a smaller constant):
\[\E[\mu_v^4]=X\le \Par{\frac{1024}{81}+\frac34}\E[\nu_v^2]+\frac34\le 66\Par{\E[\nu_v^2]+1}.\]
\end{proof}

We next give our main miscalibration bound, inspired by the proof of Proposition 1, \cite{CollinaLNR26}.

\begin{lemma}\label{lem:124_bound}
Following Construction~\ref{construction:lb} and notation \eqref{eq:round_def}, \eqref{eq:nvdef}, 
\[\sum_{v \in \calN}\E\Brack{|\mu_v|} \ge \frac{T^{\frac 3 2}}{48\sqrt{\sum_{v \in \calN} (\E[\nu_v^2] + 1)}}.\]
\end{lemma}
\begin{proof}
By H\"older's inequality with exponents $\frac 1 3$, $\frac 2 3$, 
\begin{align*}\E\Brack{\mu_v^2} = \E\Brack{|\mu_v|^{\frac 2 3} \cdot |\mu_v|^{\frac 4 3}} \le \E\Brack{|\mu_v|}^{\frac 2 3} \E\Brack{\mu_v^4}^{\frac 1 3} 
\implies \E\Brack{|\mu_v|} \ge \frac{\E[\mu_v^2]^{\frac 3 2}}{\E[\mu_v^4]^{\frac 1 2}} \ge \frac{\E[\mu_v^2]^{\frac 3 2}}{9\sqrt{\E[\nu_v^2] + 1}}, \end{align*}
where we used Lemma~\ref{lem:fourthmoment} in the last step. Further, note that
\[\E[\mu_v^2] = \E\Brack{\Par{\sum_{t \in [T]} \ind_{r_t = v}(q_t - y_t)}^2} = \E\Brack{\sum_{t \in [T]} \ind_{r_t = v}(q_t - y_t)^2} \ge \frac 3 {16}\E\Brack{\nu_v},\]
where the last inequality is because the variance $\E[(q_t - y_t)^2] = q_t( 1 - q_t) \ge \frac 3 {16}$ under Construction~\ref{construction:lb}, and the second equality is because all cross-terms vanish: for $s < t$, the variables $\ind_{r_t=v}$ and $\ind_{r_s=v}(q_s-y_s)$ are $\calF_t$-measurable, and hence
\[\E\Brack{\ind_{r_t = v}\ind_{r_s = v} (q_t - y_t)(q_s - y_s)} = \E\Brack{\ind_{r_t = v}\ind_{r_s = v}(q_s-y_s)\E\Brack{q_t-y_t\mid\calF_t}} = 0.\]
Combining the above displays, we have shown that for all $v \in \calN$,
\begin{equation}\label{eq:per_v_mubound}\E\Brack{|\mu_v|} \ge \frac{\E[\nu_v]^{\frac 3 2}}{48\sqrt{\E[\nu_v^2] + 1}} .\end{equation}
Finally, letting $\alpha_v \defeq \E[\nu_v]$ and $\beta_v \defeq \E[\nu_v^2]+1$, H\"older's inequality gives
\begin{align*}
T = \sum_{v \in \calN} \alpha_v = \sum_{v \in \calN} \Par{\frac{\alpha_v^{\frac 3 2}}{\beta_v^{\frac 1 2}}}^{\frac 2 3} \beta_v^{\frac 1 3} \le \Par{\sum_{v  \in \calN} \frac{\alpha_v^{\frac 3 2}}{\beta_v^{\frac 1 2}}}^{\frac 2 3}\Par{\sum_{v \in \calN} \beta_v}^{\frac 1 3}
\end{align*}
so that
\[\sum_{v \in \calN} \frac{\E[\nu_v]^{\frac 3 2}}{\sqrt{\E[\nu_v^2] + 1}} \ge \frac{T^{\frac 3 2}}{\sqrt{\sum_{v \in \calN} (\E[\nu_v^2] + 1)} }.\]
Summing \eqref{eq:per_v_mubound} across all $v \in \calN$, and using the above bound, gives the desired claim.
\end{proof}

To make Lemma~\ref{lem:124_bound} more usable, we next upper bound $\sum_{v \in \calN} \E[\nu_v^2]$.

\begin{lemma}\label{lem:nuv_upper}
Following Construction~\ref{construction:lb} and notation \eqref{eq:round_def}, \eqref{eq:nvdef},
\[\sum_{v \in \calN} \E[\nu_v^2] \le 128\Par{\Par{\eps T^2 + \frac 1 \eps}+ \Par{\frac {T} \eps + \frac 1 {\eps^2}}\E\Brack{\sum_{t \in [T]} (q_t - r_t)^2}}.\]
\end{lemma}
\begin{proof}
We will prove the stronger statement
\begin{equation}\label{eq:pathwise_bound}\sum_{v \in \calN} \nu_v^2 \le 128\Par{\Par{\eps T^2 + \frac 1 \eps}+ \Par{\frac {T} \eps + \frac 1 {\eps^2}}\sum_{t \in [T]} (q_t - r_t)^2}\end{equation}
for any realization of the $\{q_t, r_t\}_{t \in [T]}$. Taking expectations then gives the claim. 

To prove \eqref{eq:pathwise_bound}, we set up some helpful notation. For all $v \in \calN$, let $c_v \in \{0, L - 1, L\}$ be the number of times that $q_t = v$ in Construction~\ref{construction:lb}, where $L \defeq \lceil \frac T m \rceil$ is the number of loops. Let
\[a_v \defeq \max\Brace{\nu_v - c_v, 0},\quad b_v \defeq \sum_{t \in [T]} (q_t - r_t)^2\ind_{r_t = v}\]
respectively denote the ``overflow'' of $v$ (the number of excess times it was selected as $r_t$ compared to $q_t$), and the unscaled contribution of $v$ to the second term in \eqref{eq:pathwise_bound}.

The key observation is that of the $a_v$ excess times $r_t = v$ was chosen, at most $2L$ of the corresponding $q_t$ can satisfy $|q_t - r_t| \le \eps$, i.e., the ones from neighboring buckets $q_t \in \{r_t - \eps, r_t + \eps\}$. Similarly, at most $2L$ more can satisfy $|q_t - r_t| \le 2\eps$, and so on. Thus,
\[b_v \ge \sum_{i \in [a_v]} \Par{\eps \left\lceil\frac{i}{2L}\right\rceil}^2 \ge \sum_{i \in [a_v]} \frac{\eps^2 i}{2L} \ge \frac{\eps^2 a_v^2}{4L}.\]
Thus we expand
\begin{align*}
\nu_v^2 \le \Par{a_v + c_v}^2 \le 2c_v^2 + 2a_v^2 \le 2L^2 + \frac{8L b_v}{\eps^2}.
\end{align*}
To conclude the proof we sum over all $v \in \calN$. If $L = 1$ then \eqref{eq:pathwise_bound} clearly holds. Otherwise, observe that $m \in [\frac 1 {4\eps}, \frac 1 \eps]$ for $\eps \in (0, \frac 1 2)$. If $L \ge 2$, we have $L \le \frac{2T}{m} \le 8\eps T$, so
\[\sum_{v \in \calN} \nu_v^2 \le \frac{2L^2}{\eps} + \frac{8L}{\eps^2} \sum_{t \in [T]}(q_t - r_t)^2 \le 128\eps T^2 + \frac{64 T}{\eps}\sum_{t \in [T]}(q_t - r_t)^2.\]
\end{proof}

Finally, we can conclude with our main lower bound.

\begin{theorem}\label{thm:lb}
In the setting of Construction~\ref{construction:lb}, no $p_{[T]}$ can be $(\eps, \eps^2)$-recalibrated against $(q_{[T]}, y_{[T]})$ under $\lsq$, unless, for an appropriate constant, $T = \Omega(\eps^{-3})$.
\end{theorem}
\begin{proof}
We will instead prove that, following notation \eqref{eq:round_def}, it is impossible to obtain
\begin{equation}\label{eq:round_recal}\E\Brack{\calK_1\Par{r_{[T]}, y_{[T]}}} \le 2\eps,\quad \E\Brack{\calB_2\Par{r_{[T]}, y_{[T]}} - \calB_2\Par{q_{[T]}, y_{[T]}}} \le 4\eps^2,\end{equation}
unless $T = \Omega(\eps^{-3})$,
from which the conclusion follows by applying Lemmas~\ref{lem:round_cal_err} and~\ref{lem:round_loss_err}. Assume for contradiction that \eqref{eq:round_recal} were possible. By the bias-variance decomposition, the second bound in \eqref{eq:round_recal} implies $\E[\sum_{t \in [T]}(q_t-r_t)^2]\le 4\eps^2 T$. Lemma~\ref{lem:nuv_upper} then gives
	\begin{align*}
	\sum_{v \in \calN} \E\Brack{\nu_v^2}
	&\le 128\Par{\eps T^2 + \frac 1 \eps + 4\eps T^2 + 4T}\\
	&\le 896\Par{\eps T^2 + \frac 1 \eps},
	\end{align*}
	where we used $2T \le \eps T^2 + \frac 1 \eps$. Since $|\calN|\le 1/\eps$, we also have $\sum_{v \in \calN}(\E[\nu_v^2]+1)\le 900(\eps T^2 + 1/\eps)$. Plugging this into Lemma~\ref{lem:124_bound},
	\begin{equation}\label{eq:contradict_1}\frac 1 T \sum_{v \in \calN} \E\Brack{|\mu_v|} \ge \frac{\sqrt T}{1440\sqrt{\eps T^2 + \frac 1 \eps}}.\end{equation}
On the other hand, Lemma~\ref{lem:lower_k1_r} and the bounds in \eqref{eq:round_recal} imply
\[\frac 1 T \sum_{v \in \calN} \E\Brack{|\mu_v|} \le 2\eps + \sqrt{4\eps^2} = 4\eps. \]
We conclude by casework on the denominator of \eqref{eq:contradict_1}. If $\eps T^2 \le \frac 1 \eps \iff T \le \frac 1 \eps$, then
	\[4\eps \ge \frac{\sqrt T}{1440\sqrt{\frac 2 \eps}} \implies \eps \ge \frac{T}{32\cdot 1440^2}\]
is lower bounded by a constant, at which point every $T = \Omega(\eps^{-3})$. Otherwise, if $\eps T^2 \ge \frac 1 \eps$, 
	\[4\eps \ge \frac{\sqrt T}{1440\sqrt{2\eps T^2}} \implies T = \Omega(\eps^{-3}).\]
\end{proof}

\begin{corollary}\label{cor:l2_lower}
In the setting of Construction~\ref{construction:lb}, no $p_{[T]}$ can be $(\eps, \eps^2)$-$\calK_2$-recalibrated under $\lsq$, unless, for an appropriate constant, $T=\Omega(\eps^{-3})$.
\end{corollary}
\begin{proof}
For every prediction-label sequence, $\calK_1(p_{[T]},y_{[T]})\le \calK_2(p_{[T]},y_{[T]})$. Thus any algorithm satisfying the stated $\calK_2$ guarantee also satisfies the $\calK_1$ guarantee ruled out by \Cref{thm:lb}.
\end{proof}

\section{Experiments}

In this section, we report on experiments comparing the loss minimization and calibration properties of four algorithms: the predictions made by a base neural network used to generate a hint sequence $q_{[T]}$ in Model~\ref{model:hsp}, and three postprocessing methods applied to this hint sequence. Specifically, the four online forecasters we compare are as follows.

\begin{enumerate}
\item(Hint).
  The unmodified sigmoid output of the trained network,
  $q_t$.
\item (Calibeating).
  The algorithm in Section 4 of~\cite{FosterH23} and Section~\ref{sec:calibeat} of this work. 
\item (Recalibration).
  Single hint recalibration via Blackwell approachability (Theorem~\ref{thm:recalibration}).
\item  (Calibeating + recalibration).
  The output of Item 2 is fed to an independent instance
  of Item 3, i.e., calibeating followed by recalibration.
\end{enumerate}
We present their Brier score \eqref{eq:brier}, $\ell_1$ calibration error $\calK_1$ \eqref{eq:l1cal}, squared root-mean-square calibration error $\calK_2^2$ \eqref{eq:l2cal}, and classification accuracy (zero-one loss). For the last metric, we threshold predictions at $0.5$ to convert them to a label. \footnote{The code can be found at: \href{https://github.com/chutongyang98/Optimal-Recalibration-of-an-Online-Predictor}{https://github.com/chutongyang98/Optimal-Recalibration-of-an-Online-Predictor}}.

\subsection{Binary classification}
In this section, we consider a binary classification task. 
We use a binary subset of CIFAR-10~\cite{KrizhevskyHO09} obtained by keeping only the cat
and dog classes.  This yields $11{,}000$ training images
(of which $1{,}000$ are held out as a validation set) and $2{,}000$
test images, with cats relabeled $0$ and dogs relabeled $1$.

\paragraph{Hint model.}
As our hint model, we use a 40-layer DenseNet-BC with a single sigmoid output.  The model is trained for
100 epochs with SGD with Nesterov momentum (initial learning rate
$0.1$, weight decay $10^{-4}$, multistep schedule with drops at
epochs $50$ and $75$) under the binary cross-entropy loss.  The
trained model attains $89.6\%$ test accuracy.

\paragraph{Online stream and drift.}
The $2{,}000$ test images are reshuffled and replayed for $100$ passes,
producing an online stream of total length $T = 200{,}000$.  At step
$100{,}000$ (the midpoint, after $50$ passes) the stream enters a
``post-flip'' phase.  In this phase every label is independently flipped
$y_t \leftarrow 1 - y_t$ with probability $\beta = 0.2$, otherwise
kept.  The model's per-image predictions are unchanged across the two
phases, so the feature distribution and the forecaster's outputs are stationary while the
conditional distribution $y \mid x$ changes. The flip is realized by an independent Bernoulli stream. This label drift is meant to model a distribution shift, which could simulate a predictor that is used to label both clean and noisy data. This is a reasonable setting to evaluate recalibration and calibeating, because the corrected sequence must address this miscalibration-inducing shift.

Table~\ref{tab:cifar2-drift} reports cumulative metrics over all
$T = 200{,}000$ stream rounds.  We report the Brier score, the $\ell_1$ calibration error
$\calK_1$, the $\ell_2^2$ calibration error
$\calK_2^2$, and $0$-$1$ accuracy of thresholded labels. As theory predicts, calibeating decreases the Brier score by roughly the initial $\calK_2^2$, and empirically improves upon $\calK_1$ as well. Interestingly, recalibration also improved the Brier score slightly. Finally, calibeating + recalibration achieved the best Brier score, and a competitive $\calK_1$ and accuracy.

\begin{table}[ht]
\centering
\small
\begin{tabular}{lrrrr}
\toprule
Predictor & Brier &$\calK_1$ &  $\calK_2^2$ & Acc \\
\midrule
Hint                & $ 0.16124 \pm 7.8\mathrm{e}{-4}$& $0.15413 \pm 8.4\mathrm{e}{-4} $   &  $0.02976 \pm 3.0\mathrm{e}{-4}$ & $82.075 \pm 8.2\mathrm{e}{-2}\%$ \\
Calibeating   & $0.13865\pm 5.4\mathrm{e}{-4}$&$ 0.05771\pm 5.1\mathrm{e}{-4} $   & $ 0.00663\pm 1.0\mathrm{e}{-4}$ & $82.428\pm8.4\mathrm{e}{-2} \%$ \\
Recalibration        & $0.14708 \pm 8.3\mathrm{e}{-4}$& $ 0.01188 \pm 2.5\mathrm{e}{-4}$   & $0.00052 \pm 0.5\mathrm{e}{-4}$ & $78.597\pm  3.08\mathrm{e}{-1}\%$ \\
C + R & $0.13314 \pm 4.5\mathrm{e}{-4}$ &$0.01319 \pm 4.5\mathrm{e}{-4}$ & $9.0\mathrm{e}{-4} \pm 0.3\mathrm{e}{-4}$ & $ 82.408 \pm 8.8\mathrm{e}{-2} \%$ \\
\bottomrule
\end{tabular}
\caption{CIFAR-2 cats vs.\ dogs classification with partial concept drift (flip
probability $\beta = 0.2$ after step $100{,}000$ with $T = 200{,}000$ for $20$ random seeds).
Cumulative metrics over the full stream. Note that ``Hint'' and ``Calibeating'' are deterministic, and the randomness comes from random flipping and data order.}
\label{tab:cifar2-drift}
\end{table}
\subsection{Multiclass classification}

We include a multiclass experiment, for a classification task with $k$ classes. To generate sequences $q_{[T]}, y_{[T]}$ for use in Model~\ref{model:hsp}, we use two projection operations as suggested by \cite{GuoPSW17}. In particular, start from a hint vector $\vq \in \Delta^k$ and a one-hot label $\vy = \ve_j \in \Delta^k$. We produce a pair $q = \vq_i$, where $i \defeq \arg\max_{i \in [k]} \{\vq_i\}$ (breaking ties arbitrarily), and $y = \ind_{i = j}$. This is the same projection as is implicitly performed in \emph{confidence calibration}, where the goal is that a top-class probability of $q$ should mean this label is correct a $q$ fraction of the time.

Next, suppose that our algorithm (calibeating, recalibration, or calibeating + recalibration) processes the scalar hint $q$ and produces a prediction $p$. We lift this back up to a vector $\vp$ as follows: we take $\vp_i \gets p$, and set the other coordinates $\vp_{-i}$ to be the unique scaling of $\vq_{-i}$ so that $\vp \in \Delta^k$. When evaluating this new prediction $\vp$ according to calibration ($\calK_1$ or $\calK_2$) or Brier score ($\calB_2$), we freshly perform the binary projection operation, i.e., we relabel $p = \vp_i$ where $i = \arg\max_{i \in [k]} \vp_i$, and 
$y = \ind_{i = j}$. Due to the nonlinear nature of this operation (and the inconsistent binary labels), our theory does not technically apply to this setting. Nevertheless, we include an empirical evaluation of the performance of our postprocessing operations in this multiclass setting.

\paragraph{Dataset.}
We use CIFAR-100~\cite{KrizhevskyHO09}, a $100$-class image classification dataset on $32 \times 32$ RGB
images.  We use the standard splits: $45{,}000$ training images
(after holding out $5{,}000$ for validation), and $10{,}000$ test
images.  Test-set marginal class distribution is uniform. The $10{,}000$ test images are reshuffled and replayed for $100$ passes,
producing an online stream of total length $T = 1{,}000{,}000$. 

\paragraph{Hint model.}
As our hint model, we use a 40-layer DenseNet-BC with a $100$-way softmax output.  The model is trained
for $300$ epochs by SGD with Nesterov momentum (initial learning rate
$0.1$, weight decay $10^{-4}$, multistep schedule with drops at
$50\%$ and $75\%$ of training) under cross-entropy loss.  The trained
model attains $65.5\%$ test accuracy.  

Table~\ref{tab:cifar100} reports cumulative metrics over all
$T = 1{,}000{,}000$ stream rounds.  We report the (projected) Brier score $\calB_2$, the $\ell_1$ calibration error $\calK_1$, the $\ell_2^2$ calibration error $\calK_2^2$, and top-1 accuracy. Despite the lack of direct theoretical guarantees, the recalibration method substantially improves the calibration, and calibeating + recalibration reduces the empirical calibration of calibeating alone. Interestingly, among the three postprocessing methods, recalibration had the worst Brier score but the best top-$1$ accuracy.

\begin{table}[ht]
\centering
\small
\begin{tabular}{lrrrr}
\toprule
Predictor & Brier& $\calK_1$ & $\calK_2^2$ & Acc \\
\midrule
Hint                  & $0.55862$&$0.23330$ & $0.07473$ & $65.51\%$ \\
Calibeating   & $0.51501 \pm 1.3\mathrm{e}{-4}$ &$ 0.08476 \pm 2.6\mathrm{e}{-4}$ & $0.02286 \pm 1.6\mathrm{e}{-4}$ & $62.114\pm 2.3\mathrm{e}{-2}\%$ \\
Recalibration        & $ 0.55281\pm 1.8\mathrm{e}{-4}$&$0.03107 \pm 1.05\mathrm{e}{-3}$ & $0.00336 \pm 1.5\mathrm{e}{-4}$ & $64.379 \pm 5.1\mathrm{e}{-2}\%$ \\
C + R & $0.51240 \pm 1.4\mathrm{e}{-4}$&$ 0.06622 \pm 7.1\mathrm{e}{-4}$ & $0.01152 \pm 1.8\mathrm{e}{-4}$ & $63.549 \pm 3.9\mathrm{e}{-2} \%$ \\
\bottomrule
\end{tabular}
\caption{CIFAR-100 multi-class with $T = 1{,}000{,}000$ for $20$ random seeds. Note that ``Calibeating'' is deterministic and the randomness comes from different orderings on online data.}
\label{tab:cifar100}
\end{table}

\section*{AI Disclosure}
\phantomsection
\addcontentsline{toc}{section}{AI Disclosure}

LLMs suggested both the $\calK_2$ extension to our main $\calK_1$ recalibration result in Section~\ref{sec:l2recal}, and the lower bound argument in Section~\ref{sec:lower}, upon being prompted with the related lower bound of \cite{CollinaLNR26} as a starting point. The remaining results, including our main upper bound in Section~\ref{sec:recal}, were proven by the authors prior to consultation with LLMs. The manuscript was written solely by the authors, who take full responsibility for the organization and presentation of results derived by LLMs.

\newpage

\bibliography{ref.bib}
\bibliographystyle{alpha}

\newpage
\appendix

\end{document}